\documentclass[twoside,11pt]{article}
\usepackage{jmlr2e}

\jmlrheading{1}{****}{**-**}{4/00}{10/00}{Tadahiro Taniguchi, Toshiaki Takano and Ryo Yoshino}

\ShortHeadings{Multimodal Hierarchical Dirichlet Process-based Active Perception}{Taniguchi, Takano, , \& Yoshino}
\firstpageno{1}

\def\aref#1{({\ref{#1}})}
\usepackage{graphicx}
\usepackage{times} % assumes new font selection scheme installed
\usepackage{amsmath} % assumes amsmath package installed
\usepackage{amssymb}  % assumes amsmath package installed
\usepackage{algorithm}
\usepackage{algorithmic}

\DeclareMathOperator*{\KL}{KL}
\DeclareMathOperator*{\IG}{IG}
\DeclareMathOperator*{\argmax}{argmax}
\DeclareMathOperator*{\argmin}{argmin}
\DeclareMathOperator*{\minimize}{minimize}

\begin{document}

\title{Multimodal Hierarchical Dirichlet Process-based Active Perception}

\author{\name Tadahiro Taniguchi \email taniguchi@ci.ritsumei.ac.jp \\
       \name Toshiaki Takano \email takano@em.ci.ritsumei.ac.jp \\
       \addr Department of Human and Computer Intelligence, Ritsumeikan University, \\
       Nojihigashi 1-1-1, Kusatsu, Shiga 525-8577 Japan.
       \AND
       \name Ryo Yoshino \email yoshino@em.ci.ritsumei.ac.jp  \\
       \addr Graduate School of Information Science and Engineering, Ritsumeikan University, \\
       Nojihigashi 1-1-1, Kusatsu, Shiga 525-8577 Japan.
       }

% For research notes, remove the comment character in the line below.
% \researchnote
\editor{****}
\maketitle

\begin{abstract}
In this paper, we propose an active perception method for recognizing object categories based on the multimodal hierarchical Dirichlet process (MHDP). The MHDP enables a robot to form object categories using multimodal information, e.g., visual, auditory, and haptic information, which can be observed by performing actions on an object. 
However, performing many actions on a target object requires a long time.
In a real-time scenario, i.e., when the time is limited, the robot has to determine the set of actions that is most effective for recognizing a target object.
We propose an MHDP-based active perception method that uses the information gain (IG) maximization criterion and lazy greedy algorithm.
We show that the IG maximization criterion is optimal in the sense that the criterion is equivalent to a minimization of the expected Kullback--Leibler divergence between a final recognition state and the recognition state after the next set of actions. 
However, a straightforward calculation of IG is practically impossible.
Therefore, we derive an efficient Monte Carlo approximation method for IG by making use of a property of the MHDP. We also show that the IG has submodular and non-decreasing properties as a set function because of the structure of the graphical model of the MHDP. Therefore, the IG maximization problem is reduced to a submodular maximization problem. This means that greedy and lazy greedy algorithms are effective and have a theoretical justification for their performance. 
We conducted an experiment using an upper-torso humanoid robot and a second one using synthetic data. The experimental results show that the method enables the robot to select a set of actions that allow it to recognize target objects quickly and accurately. The results support our theoretical outcomes.
\end{abstract}

\section{Introduction}\label{sec:1}
Active perception is a fundamental component of our cognitive skills. Human infants autonomously and spontaneously perform actions on an object to determine its nature.  
The sensory information that we can obtain usually depends on the actions performed on the target object. For example, when a person finds a box placed in front of him/her, he/she cannot perceive its weight without holding the box, and he/she cannot determine its sound without hitting or shaking it. In other words, we can obtain sensory information about an object by selecting and executing actions to manipulate it.
Adequate action selection is important for recognizing objects quickly and accurately. This example about a human also holds for a robot.
An autonomous robot that moves and helps people in a living environment should also select adequate actions to recognize target objects.
For example, when a person asks an autonomous robot to bring an empty plastic bottle, the robot has to examine many objects by applying several actions (Fig.~\ref{fig:action_selection}).
The importance of this type of active perception is because our object categories are formed on the basis of multimodal information, i.e., not only visual information, but also auditory, haptic, and other information. Therefore, a computational model of the active perception should be consistently based on a computational model for multimodal object categorization and recognition. 

\begin{figure}[!tb]
\begin{center}
\includegraphics[width=100mm]{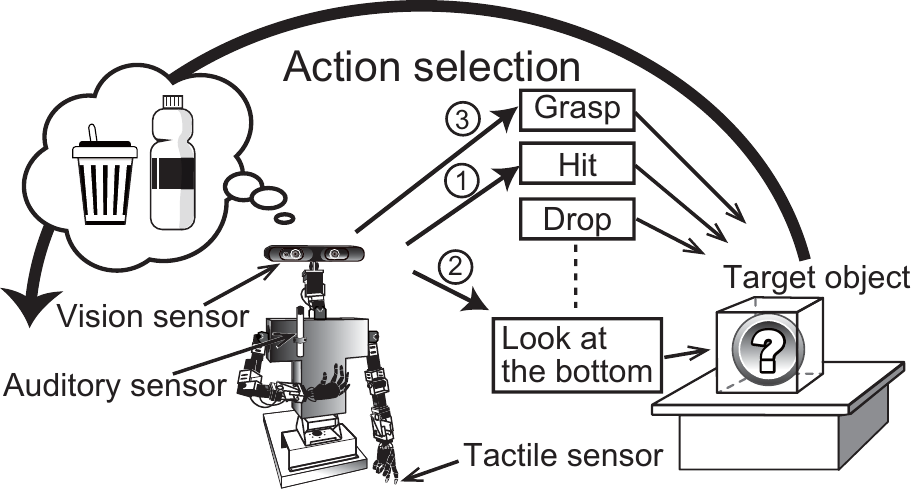}
\caption{Overview of active perception for multimodal object category recognition.}
\label{fig:action_selection}
\end{center}
\end{figure}

This paper considers the active perception problem for multimodal object recognition. Specifically, we adopt the multimodal hierarchical Dirichlet process (MHDP) proposed by Nakamura et al. (2011b) as a representative computational model for multimodal object categorization. We develop an active perception method based on the MHDP. The MHDP is a sophisticated, fully Bayesian probabilistic model for multimodal object categorization\nocite{MHDP}. It is a multimodal extension of hierarchical Dirichlet process (HDP)~\cite{hdp}, which is a nonparametric Bayesian extension of latent Dirichlet allocation (LDA)~\cite{Blei2003}, which in turn was originally proposed for document-word clustering. Nakamura et al. (2011b) showed that the MHDP enables a robot to form object categories using multimodal information, i.e., visual, auditory, and haptic information, in an unsupervised manner. Because of the nature of Bayesian nonparametrics, the MHDP can estimate the number of object categories as well.%~\cite{MHDP}.

In spite of the wide range of studies about active perception and multimodal categorization for robots, active perception methods, i.e., action selection methods for perception for multimodal categorization have not been sufficiently explored from a theoretical viewpoint (see Section~\ref{sec2}).  
This paper describes a new MHDP-based active perception method for multimodal object recognition based on object categories formed by a robot itself. We found that an active perception method that has a good theoretical nature 
can be derived by taking the MHDP as a robot's multimodal categorization method. 

In this study, we define the active perception problem in the context of unsupervised multimodal object categorization as follows.
\begin{itemize}
\item Which set of actions should a robot take to recognize a target object as accurately as possible under the constraint that the number of actions is restricted?
\end{itemize}
Our MHDP-based active perception method uses an information gain (IG) maximization criterion, Monte Carlo approximation, and the lazy greedy algorithm.
In this paper, we show that the MHDP provides the following three advantages for deriving an efficient active perception method.

\begin{enumerate}
\item The {\em IG maximization criterion} is {\em optimal} in the sense that a selected set of actions minimizes the expected Kullback--Leibler (KL) divergence between the final posterior distribution estimated using the information regarding all modalities and the posterior distribution of the category estimated using the selected set of actions.
\item An efficient {\em Monte Carlo approximation} method for IG can be derived.
\item The IG has a {\em submodular} and non-decreasing property as a set function. Therefore, for performance, the greedy and lazy greedy algorithms are guaranteed to be near-optimal strategies.
\end{enumerate}
Although the above desirable properties are due to the theoretical characteristics of the MHDP, this has never been pointed out in previous studies. 

The main contributions of this paper are that we present the above three properties of the MHDP clearly, develop an MHDP-based active perception method, and show its effectiveness through experiments using a upper-torso humanoid robot and synthetic data. 

The proposed active perception method can be used for general purposes, i.e., not only for robots but also for other target domains to which the MHDP can be applied. In addition, The proposed method can be easily extended for multimodal latent Dirichlet allocation (MLDA), which is a multimodal extension of latent Dirichlet allocation (LDA)~\cite{Nakamura2009,Blei2003}, and other multimodal categorization methods with similar graphical models. 
However, in this paper, we focus on the MHDP and the robot active perception scenario, and explain our method on the basis of this task.

The remainder of this paper is organized as follows. Section~\ref{sec2} describes the background and work related to our study. Section~\ref{sec3} briefly introduces the MHDP, proposed by Nakamura et al. (2011b), which enables a robot to obtain an object category by fusing multimodal sensor information in an unsupervised manner.
%.~\cite{MHDP}.
Section~\ref{sec4} describes our proposed action selection method. Section~\ref{sec5} discusses the effectiveness of the action selection method through experiments using an upper-torso humanoid robot. 
Section~\ref{sec6} describes a supplemental experiment using synthetic data. 
Section~\ref{sec7} concludes this paper.

\section{Background and Related Work}\label{sec2}
In this section, we describe background and related work of this paper. 
\subsection{Multimodal Categorization}
The human capability for object categorization is a fundamental topic in cognitive science~\cite{Barsalou1999}.
In the field of robotics, adaptive formation of object categories that considers a robot's embodiment, i.e., its sensory-motor system, is gathering attention as a way to solve the symbol grounding problem~\cite{harnad1990symbol,Taniguchi2015SER}.

Recently, various computational models and machine learning methods for multimodal object categorization have been proposed in artificial intelligence, cognitive robotics, and related research fields~\cite{Celikkanat,Sinapov,Natale,Araki2012,Ando,Nakamura2007,Nakamura2009,Nakamura2011,MHDP,Nakamura2014,Griffith2012,Iwahashi2010,Roy2002a,Sinapov2014}.
For example, Sinapov \& Stoytchev (2011) proposed a graph-based multimodal categorization method that allows a robot to recognize a new object by its similarity to a set of familiar objects.~\nocite{Sinapov}. 
They also built a robotic system that categorizes 100 objects from multimodal information in a supervised manner~\cite{Sinapov2014}. 
 Celikkanat et al. (2014) modeled the context in terms of a set of concepts that allow many-to-many relationships between objects and contexts using latent Dirichlet allocation.~\nocite{Celikkanat}.

\begin{figure}[tb!p]
\centering
\includegraphics[width=60mm]{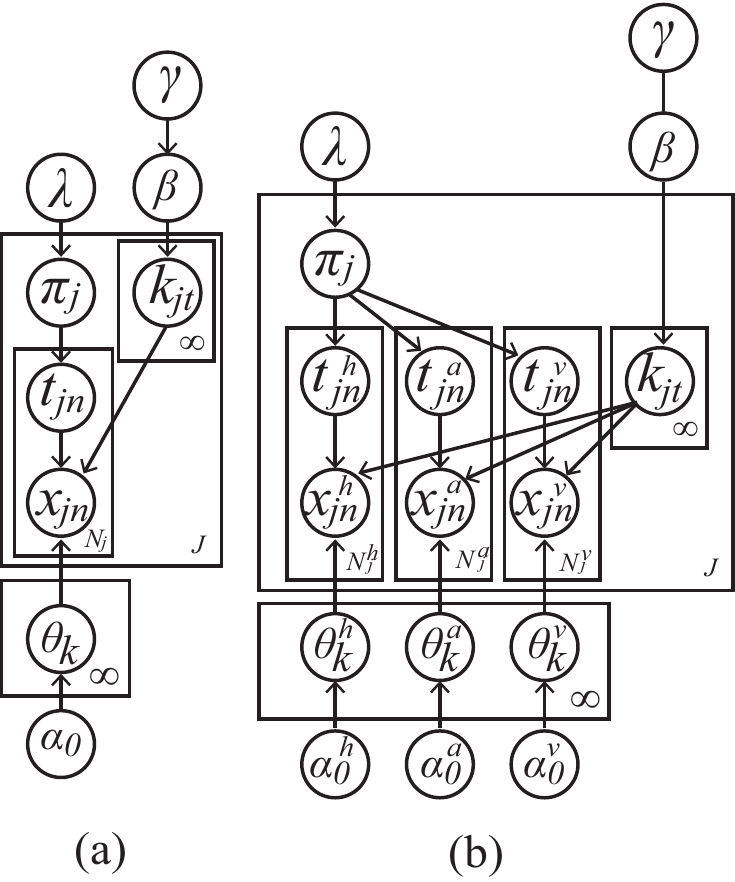}
\caption{Graphical representation of HDP~\cite{Sudderth2006} and MHDP~\cite{MHDP}.}
\label{fig:hdp-mhdp}
\end{figure}

Of these, a series of statistical multimodal categorization methods for autonomous robots have been proposed by extending LDA, i.e., a topic model~\cite{Araki2012,Ando,Nakamura2007,Nakamura2009,Nakamura2011,MHDP,Nakamura2014}. All these methods are Bayesian generative models,
and the MHDP is a representative method of this series~\cite{MHDP}.
The MHDP is an extension of the HDP, which was proposed by Teh et al. (2006), and
the HDP is a nonparametric Bayesian extension of LDA~\cite{Blei2003}.
A graphical model of the HDP is shown in Fig.~\ref{fig:hdp-mhdp}(a).
Concretely, the graphical model of the MHDP has multiple types of emissions that correspond to various sensor data obtained through various modality inputs, as shown in Fig.~\ref{fig:hdp-mhdp}(b).
In the HDP, observation data are usually represented as a bag-of-words (BoW). In contrast, the observation data in the MHDP use bag-of-features (BoF) representations for multimodal information. Latent variables $t_{jn}$ are regarded as indicators of  {\em topics} in the HDP, which correspond to {\em object categories} in the MHDP.
Nakamura et al. (2011b) showed that the MHDP enables a robot to categorize a large number of objects in a home environment into categories that are similar to human categorization results.%~\cite{MHDP}.

To obtain multimodal information, a robot has to perform actions and interact with a target object in various ways, e.g., grasping, shaking, or rotating the object. If the number of actions and types of sensor information increase, multimodal categorization and recognition can require a longer time.
In most practical cases, the execution of an action by a robot takes longer than it does for a human for mechanical and security reasons. In many cases, one action can take longer than 30 seconds, although that depends on each particular robotic system. 
When the recognition time is a constraint and/or if quick recognition is required, it becomes important for a robot to select a small number of actions that are effective for accurate recognition.
Action selection for recognition is often called active perception. 
However, an active perception method for the MHDP has not been proposed. This paper aims to provide an active perception method for the MHDP.

\subsection{Active Perception}
Generally, active perception is one of the most important cognitive capabilities of humans. 
From an engineering viewpoint, active perception has many specific tasks, e.g.,  localization, mapping, navigation, object recognition, object segmentation, and self--other differentiation.

Historically, active vision, i.e., active visual perception, has been studied as an important engineering problem in computer vision.
Roy et al. (2004) presented a comprehensive survey of active three-dimensional object recognition\nocite{DuttaRoy2004a}.
For example, Borotshnig et al. (2000) proposed an active vision method in a parametric eigenspace to improve the visual classification results\nocite{Borotschnig2000}.
Denzler et al. (2002) proposed an information theoretic action selection method to gather information that conveys the true state of a system through an active camera\nocite{Denzler2002}.
They used the mutual information (MI) as a criterion for action selection.
Krainin et al. (2011) developed an active perception method in which a mobile robot manipulates an object to build a three-dimensional surface model of it\nocite{Krainin2011a}.
Their method uses the IG criterion to determine when and how the robot should grasp the object.

Modeling and/or recognizing a single object as well as modeling a scene and/or segmenting objects are also important tasks in the context of robotics.
Eidenberger et al. (2010) proposed an active perception planning method for scene modeling in a realistic environment\nocite{Eidenberger2010a}. Hoof et al. (2012) proposed an active scene exploration method that enables an autonomous robot to efficiently segment a scene into its constituent objects by interacting with the objects in an unstructured environment\nocite{VanHoof2012a}. They used IG as a criterion for action selection.
InfoMax control for acoustic exploration was proposed by Rebguns et al. (2011)\nocite{Rebguns2011}. 

Localization, mapping, and navigation are also targets of active perception.
Velez et al. (2012) presented an online planning algorithm that enables a mobile robot to generate plans that maximize the expected performance of object detection\nocite{Velez2012}. 
Burgard et al. (1997) proposed an active perception method for localization\nocite{Fox}. Action selection is performed by maximizing the weighted sum of the expected entropy and expected costs.
To reduce the computational cost, they only consider a subset of the next locations.
Roy et al. (1999) proposed a coastal navigation method for a robot to generate trajectories for its goal by minimizing the positional uncertainty at the goal\nocite{Roy}.
Stachniss et al. (2005) proposed an information-gain-based exploration method for mapping and localization.\nocite{Stachniss}.
Correa et al. proposed an active perception method for a mobile robot with a visual sensor mounted on a pan-tilt mechanism to reduce localization uncertainty.
They used the IG criterion, which was estimated using a particle filter.

In addition, various studies on active perception by a robot have been conducted~\cite{Gouko,Saegusa2011,Ji2006,Tuci2010,Natale,Schneider2009a,Sushkov2012,Hogman2013,Ivaldi2014a,Fishel2012,Pape2012}.
In spite of a large number of contributions about active perception, 
 few theories of active perception for multimodal object category recognition have been proposed. In particular, an MHDP-based active perception method has not yet been proposed, although the MHDP-based categorization method and its series have obtained many successful results and extensions. 

In machine learning, {\em active learning} is a well-defined terminology. Active learning algorithms select an unobserved input datum and ask a user (labeler) to provide a training signal (label) in order to reduce uncertainty as quickly as possible~\cite{Cohn1996,Settles2012,Muslea2006}. These algorithms usually assume a supervised learning problem. This problem is related to the problem in this paper, but is fundamentally different.

\subsection{Active perception for multimodal categorization} 
Sinapov et al. (2014) investigated multimodal categorization and active perception by making a robot perform 10 different behaviors; obtain visual, auditory, and haptic information; explore 100 different objects, and classify them into 20 object categories\nocite{Sinapov2014}. In addition, they proposed an active behavior selection method based on confusion matrices. They reported that the method was able to reduce the exploration time by half by dynamically selecting the next exploratory behavior.
However, their multimodal categorization is performed in a supervised manner, and the theory of active perception is still heuristic. The method does not have theoretical guarantees of performance.

IG-based active perception is popular, as shown above, but the theoretical justification for using IG in each task is often missing in many robotics papers. 
Moreover, in many cases, IG cannot be evaluated directly, reliably, or accurately. When one takes an IG criterion-based approach, how to estimate the IG is an important problem.
In this study, we focus on MHDP-based active perception and develop an efficient near-optimal method based on firm theoretical justification.

\section{Multimodal Hierarchical Dirichlet Process for Statistical Multimodal Categorization}\label{sec3}
We assume that a robot forms object categories using the MHDP from multimodal sensory data. In this section, we briefly introduce the MHDP on which our proposed active perception method is based~\cite{MHDP}.  
The MHDP assumes that an observation node in its graphical model corresponds to an action and its corresponding modality. Nakamura et al. (2011b) employed three observation nodes in their graphical model, i.e., haptic, visual, and auditory information nodes. 
Three actions, i.e., grasping, looking around, and shaking, correspond to these modalities, respectively.
However, the MHDP can be easily extended to a model with additional types of sensory inputs. It is without doubt that autonomous robots will also gain more types of action for perception. 
For modeling more general cases, an MHDP with $M$ actions is described in this paper.
A more general graphical model of the MHDP than in Fig.~\ref{fig:hdp-mhdp} is illustrated in Fig.~\ref{fig:g-mhdp}.

%---------------------------------------------------------------------
\begin{figure}[tb!p]
\centering

\includegraphics[width=70mm]{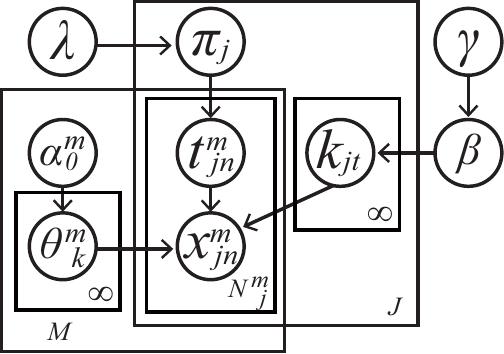}
\caption{Graphical representation of an MHDP with $M$ modalities corresponding to actions for perception.}
\label{fig:g-mhdp}
\end{figure}
%---------------------------------------------------------------------

The index $m \in \mathbf{M}$ $(\#(\mathbf{M}) = M)$ in Fig. \ref{fig:g-mhdp} represents the type of information that corresponds to an action-modality perception pair, e.g., hitting an object to obtain its sound, grasping an object to test its shape and hardness, or looking at all of an object by rotating it.
The observation $x_{jn}^{m} \in \mathbf{X}^m$ is the $m$-th modality's $n$-th feature for the $j$-th target object. The observation $x_{jn}^{m}$ is assumed to be drawn from a categorical distribution whose parameter is $\theta_{k}^{m}$, {where $k$ is an index of a latent topic}. Parameter $\theta_{k}^{m}$ is assumed to be drawn from the Dirichlet prior distribution whose parameter is $\alpha_{0}^{m}$. The MHDP assumes that a robot obtains each modality's sensory information as a BoF representation.

Similarly to the generative process of the original HDP~\cite{hdp}, the generative process of the MHDP can be described as a Chinese restaurant franchise (CRF). The learning and recognition algorithms are both derived using Gibbs sampling. In its learning process, the MHDP estimates a latent variable $t^m_{jn}$ for each feature of the $j$-th object and a topic index $k_{jt}$ for each latent variable $t$. The combination of latent variable and topic index corresponds to a topic in LDA~\cite{Blei2003}. Using the estimated latent variables, the categorical distribution parameter $\theta_{k}^{m}$ and topic proportion of the $j$-th object $\pi_j$ are drawn from the posterior distribution.

The selection procedure for latent variable $t^m_{jn}$ is as follows. The prior probability that $x_{jn}^{m}$ selects $t$ is
\begin{align}
& P(t^m_{jn} = t | \lambda)= \left\{
 \begin{array}{l}
\frac{\sum_{m}w^{m}N_{jt}^{m}}{\lambda+\sum_{m}w^{m}N^{m}_{j}-1},\hspace{10mm}(t=1,\cdots,T_{j}),\nonumber\\
\frac{\lambda}{\lambda+\sum_{m}w^{m}N^{m}_{j}-1},\hspace{10mm}(t=T_{j}+1),
\end{array} \right. \nonumber
\end{align}
where $w^{m}$ is a weight for the $m$-th modality,  $N^m_{jt}$ is the number of $m$-th modality observations that are allocated to $t$ in the $j$-th object, and $\lambda$ is a hyperparameter.
In the Chinese restaurant process, if the number of observed features $N_{jt}=\sum_m w^{m}N^{m}_{jt}$ that are allocated to $t$ increases, the probability at which a new observation is allocated to the latent variable $t$  increases.
Using the prior distribution, the posterior probability that observation $x^{m}_{jn}$ is allocated to the latent variable $t$ becomes
\begin{align}
 P(t^{m}_{jn}&=t| X^{m},\lambda)\propto P(x^{m}_{jn} | X^{m}_{k=k_{jt}})P(t^m_{jn} = t |\lambda)\nonumber \\
&= \left\{
 \begin{array}{l}
 P(x^{m}_{jn} | X^{m}_{k=k_{jt}})\frac{\sum_{m}w^{m}N^{m}_{jt}}{\lambda+\sum_{m}w^{m}N^{m}_{j}-1},\hspace{1mm}(t=1,\cdots,T_{j}),\nonumber\\
 P(x^{m}_{jn} | X^{m}_{k=k_{jt}})\frac{\lambda}{\lambda+\sum_{m}w^{m}N^{m}_{j}-1},\hspace{1mm}(t=T_{j}+1),\nonumber
 \end{array}
 \right.
\end{align}
where $N_{j}^{m}$ is the number of the $m$-th modality's observations about the $j$-th object. The observations that correspond to the $m$-th modality and have the $k$-th topic in any object are represented by $X_{k}^{m}$.

In the Gibbs sampling procedure, a latent variable for each observation is drawn from the posterior probability distribution. If $t = T_{j} + 1$, a new observation is allocated to a new latent variable.
The dish selection procedure is as follows.
The prior probability that the $k$-th topic is allocated on the $t$-th latent variable becomes
\begin{align}
 P(k_{jt} = k | \gamma)=\left\{
 \begin{array}{l}
\frac{M_{k}}{\gamma+M-1},\hspace{10mm}(k=1,\cdots,K),\\
\frac{\gamma}{\gamma+M-1},\hspace{10mm}(k=K+1),\nonumber
\end{array} \right.
\end{align}
where $K$ is the number of topic types, and $M_{k}$ is the number of latent variables on which the $k$-th topic is placed. 
Therefore, the posterior probability that the $k$-th topic is allocated on the $t$-th latent variable becomes
\begin{align}
 P(k_{jt}=k | X,\gamma) &= P(X_{jt} | X_{k})P(k_{jt}=k|\gamma)\nonumber \\
& = \left\{
 \begin{array}{l}
 P(X_{jt} | X_{k})\frac{M_{k}}{\gamma+M-1},\hspace{1mm}(k=1,\cdots,K),\\
 P(X_{jt} | X_{k})\frac{\gamma}{\gamma+M-1},\hspace{1mm}(k=K+1).\nonumber
 \end{array} \right.
\end{align}
A topic index for the latent variable $t$ for the $j$-th object is drawn using the posterior probability, where $\gamma$ is a hyperparameter.
If $k = K+1$, a new topic is placed on the latent variable.

By sampling $t^m_{jn}$ and $k_{jt}$, the Gibbs sampler performs probabilistic object clustering:
\begin{eqnarray}
t^{m}_{jn} &\sim& P(t^{m}_{jn} | X^{-mjn},\lambda),
\label{eq:t}\\
k_{jt} &\sim& P(k_{jt} | X^{-jt},\gamma),
\label{eq:k}
\end{eqnarray}
where $X^{-mjn} = X^{m}_{jn} \setminus \{x^{m}_{jn}\}$, and  $X^{-jt} = X_{t}\setminus X_{jt} $.
By sampling $t^{m}_{jn}$ for each observation in every object using (\ref{eq:t}) and sampling $k_{jt}$ for each latent variable $t$ in every object using (\ref{eq:k}), all of the latent variables in the MHDP can be inferred.

If $t^{m}_{jn}$ and $k_{jt}$ are given, the probability that the $j$-th object is included in the $k$-th category becomes
\begin{eqnarray}
P(k | X_{j}) = \frac{\Sigma^{T_{j}}_{t=1}\delta_{k}(k_{jt})\sum_{m}w^{m}N^{m}_{jt}}{\sum_{m}w^{m}N^{m}_{j}},
\label{eq:p}
\end{eqnarray}
where $X_{j}=\cup_m X^m_{j}$, $w^{m}$ is the weight for the $m$-th modality and $\delta_{a}(x)$ is a delta function.

When a robot attempts to recognize a new object after the learning phase, the probability that feature $x^m_{jn}$ is generated from the $k$-th topic becomes
\begin{equation}
P(x^{m}_{jn} | X_{k}^{m})=\frac{w^{m}N_{kx^{m}_{jn}}^{m}+\alpha_{0}^{m}}{w^{m}N_{k}^{m}+d^{m}\alpha_{0}^{m}},\nonumber
\end{equation}
where $d^m$ denotes the dimension of the $m$-th modality input.
Topic $k_t$ allocated to $t$ for a new object is sampled from
\begin{equation}
k_t \sim P(k_{jt}=k|X,\gamma) \propto P(X_{jt} | X_k)\frac{\gamma}{\gamma +M - 1}.\nonumber
\end{equation}
These sampling procedures play an important role in the Monte Carlo approximation of our proposed method (see Section 4.2.) 

For a more detailed explanation of the MHDP, please refer to Nakamura et al. (2011b)\nocite{MHDP}.
Basically, a robot can autonomously learn object categories and recognize new objects using the multimodal categorization procedure described above. The performance and effectiveness of the method was evaluated in the paper.

\section{Active Perception Method}\label{sec4}
In this section, we describe active perception method based on the MHDP. 
\subsection{Basic Formulation}
A robot should have already conducted several actions and obtained information from several modalities when it attempts to select next action set for recognizing a target object. For example, visual information can usually be obtained by looking at the front face of the $j$-th object from a distance before interacting with the object physically.
We assume that a robot has already obtained information corresponding to a subset of modalities $\mathbf{m_o}_j \subset  \mathbf{M} $. When a robot faces a new object and has not obtained any information, $\mathbf{m_o}_j = \emptyset $.

The purpose of object recognition in multimodal categorization is different from conventional supervised learning-based pattern recognition problems. In supervised learning, the recognition result is evaluated by checking whether the output is same as the truth label. However, in unsupervised learning, there are basically no truth labels. Therefore, the performance of active perception should be measured in a different manner. 

The action set the robot selects is described as $\mathbf{A}= \{a_1, a_2, \ldots ,  a_{\#(\mathbf{A})}\}$ $ \in \mathbf{2}^{\mathbf{M} \setminus \mathbf{m_o}_{j}}$, where $ \mathbf{2}^{\mathbf{M} \setminus \mathbf{m_o}_{j}}$ is a family of subsets of ${\mathbf{M} \setminus \mathbf{m_o}_{j}}$, i.e., $\mathbf{A} \subset \mathbf{M}\setminus \mathbf{m_o}_{j}$ and $a_i \in \mathbf{M}\setminus \mathbf{m_o}_{j}$.
We consider an effective action set for active perception to be one that largely reduces the distance between the final recognition state after the information from all modalities $\mathbf{M}$ is obtained and the recognition state after the robot executes the selected action set $\mathbf{A} $. 
The recognition state is represented by the posterior distribution $P( z_j | X_j^{\mathbf{m_o}_j \cup \mathbf{A}})$. Here, $z_j = \{\{k_{jt}\}_{1 \le t \le T_j},$ $ \{t_{jn}^m\}_{m \in \mathbf{M}, 1 \le n \le N^m_j}\}$ is a latent variable representing the $j$-th object's topic information, where $X_j^\mathbf{A}  = \cup_{m\in \mathbf{A} } X_j^m $,
$X^m_j = \{ x^m_{j1},\dots,x^m_{jn},$ $\dots,x^m_{jN_j^m} \}$. 
Probability $P( z_j | X_j^{\mathbf{m_o}_j \cup \mathbf{A}})$ represents the posterior distribution related to the object category after taking actions $\mathbf{m_o}_j$ and $\mathbf{A}$.

The final recognition state, i.e., posterior distribution over latent variables after obtaining the information from all modalities $\mathbf{M}$, becomes $P( z_j | X_j^{\mathbf{M}} )$.
The purpose of active perception is to select a set of actions that
can estimate the posterior distribution most accurately.  
When $L$ actions can be executed, if we employ KL divergence as the metric of the difference between the two probability distributions,
\begin{equation}
\minimize_{\mathbf{A} \in \mathbf{F}^{\mathbf{m_o}_j}_L} \KL \big( P( z_j | X_j^{\mathbf{M}} ) , P( z_j | X^{\mathbf{m_o}_j \cup \mathbf{A}}_j ) \big) \label{eq:goal_KL_det}
\end{equation}
is a reasonable evaluation criterion for realizing effective active perception, where $\mathbf{F}_L^{\mathbf{m_o}_j} = \{ \mathbf{A}  | \mathbf{A}  \subset \mathbf{M} \setminus \mathbf{m_o}_{j}, \#({\mathbf{A}})\le L \}$ is a feasible set of actions.

However, neither the true $X_j^{\mathbf{M}}$ nor $X^{\mathbf{m_o}_j \cup \mathbf{A}}_j $ can be observed before taking $\mathbf{A}$ on the $j$-th target object, and hence cannot be used at the moment of action selection. Therefore, a rational alternative for the evaluation criterion is the expected value of the KL divergence at the moment of action selection:
\begin{equation}
\minimize_{\mathbf{A} \in \mathbf{F}^{\mathbf{m_o}_j}_L} \mathbb{E}_{ X_j^{\mathbf{M}\setminus \mathbf{m_o}_j} | X^{\mathbf{m_o}_j}_j } [ \KL \big( P( z_j | X_j^{\mathbf{M}} ) , P( z_j | X_j^{\mathbf{m_o}_j  \cup  \mathbf{A}    }   ) \big)]. \label{eq:exp_kl_crt}
\end{equation}

Here, we propose to use the IG maximization criterion to select the next action set for active perception:
\begin{align}
\mathbf{A}_j^* &= \argmax_{\mathbf{A} \in \mathbf{F}^{\mathbf{m_o}_j}_L}
\IG (z_j ;X_j^\mathbf{A} | X_j^{\mathbf{m_o}_j} )
\label{eq:IG_max}\\
 &= \argmax_{\mathbf{A} \in \mathbf{F}^{\mathbf{m_o}_j}_L}
 \mathbb{E}_{ X_j^\mathbf{A} | X_j^{\mathbf{m_o}_{j} }}
 [ \KL \big( P( z_j | X_j^{\mathbf{m_o}_j \cup \mathbf{A}}) , P( z_j | X_j^{\mathbf{m_o}_j } ) \big) ],  \label{eq:KL_max}
\end{align}
where $\IG (X;Y |Z)$ is the IG of $Y$ for $X$, which is calculated on the basis of the probability distribution commonly conditioned by $Z$ as follows:
\begin{equation}
\IG (X;Y |Z) =\KL \big(P(X,Y|Z) , P(X|Z)P(Y|Z) \big).\nonumber
\end{equation}
By definition, the expected KL divergence is the same as $\mathrm{IG} (X; Y)$. The definition of $\IG$ and its relation to KL divergence are as follows.
\begin{eqnarray}
\IG (X;Y)&=& H(X ) - H(X | Y) \nonumber\\
&=& \KL \big(P(X,Y) , P(X)P(Y) \big)\nonumber\\
&=& \mathbb{E}_Y [\KL \big(P(X|Y),P(X) \big) ].\nonumber
\end{eqnarray}
The optimality of the proposed criterion \aref{eq:IG_max} is supported by Theorem~\ref{theo:1}.
\begin{theorem}\label{theo:1}
The set of next actions $\mathbf{A} \in \mathbf{F}^{\mathbf{m_o}_j}_L $ that maximizes the $\IG (z_j ;X_j^\mathbf{A} | X_j^{\mathbf{m_o}_j} )$
minimizes the expected KL divergence between
the posterior distribution over $z_j$ after all modality information has been observed and after $\mathbf{A}$ has been executed.
\begin{eqnarray}
\argmin_{\mathbf{A} \in \mathbf{F}^{\mathbf{m_o}_j}_L} \mathbb{E}_{ X_j^{\mathbf{M}\setminus \mathbf{m_o}_j} | X^{\mathbf{m_o}_j}_j } [ \KL \big( P( z_j | X_j^{\mathbf{M}} ) , P( z_j | X_j^{\mathbf{m_o}_j \cup   \mathbf{A}   } ) \big)] 
= \argmax_{\mathbf{A} \in \mathbf{F}^{\mathbf{m_o}_j}_L}
\IG (z_j ;X_j^\mathbf{A} | X_j^{\mathbf{m_o}_j} )\nonumber
\end{eqnarray}
\end{theorem}
\begin{proof}
See Appendix A.
\end{proof}
This theorem is essentially the result of well-known characteristics of IG (see \cite{Russo2015,MacKay2003} for example). 
This means that maximizing IG
 is the optimal policy for active perception in an MHDP-based  multimodal object category recognition task.
As a special case, when only a single action is permitted, the following corollary is satisfied.
\begin{corollary}
The next action $m\in \mathbf{M} \setminus \mathbf{m_o}_j$ that maximizes 
$\IG (z_j ;X_j^m | X_j^{\mathbf{m_o}_j} )$
minimizes the expected KL divergence between
the posterior distribution over $z_j$ after all modality information has been observed and after the action has been executed.
\begin{align}
\argmin_{m \in \mathbf{M} \setminus \mathbf{m_o}_j} \mathbb{E}_{ X^{\mathbf{M}\setminus \mathbf{m_o}_j}_j | X^{\mathbf{m_o}_j}_j } [ \KL \big( P( z_j | X_j^{\mathbf{M}} ) , P( z_j | X^{\{m\} \cup \mathbf{m_o}_j}_j ) \big)]% \nonumber\\
= \argmax_{m \in \mathbf{M} \setminus \mathbf{m_o}_{j}} \IG (z_j ;X_j^m | X_j^{\mathbf{m_o}_j} ).\label{eq:policy}
\end{align}
\end{corollary}
\begin{proof}
By substituting $\{m\}$ into $\mathbf{A} $ in Theorem \ref{theo:1}, we can obtain the corollary.
\end{proof}
Using $\IG$, the active perception strategy for the next single action is simply described as follows:
\begin{equation}
m_j^* = \argmax_{m \in \mathbf{M} \setminus \mathbf{m_o}_{j}}
\IG (z_j ;X_j^m | X_j^{\mathbf{m_o}_j} ).\label{eq:policy}
\end{equation}
This means that the robot should select the action $m_j^* $ that can obtain the $X_j^{m_j^*} $ that maximizes the IG for the recognition result $z_j$ under the condition that the robot has already observed $X_j^{\mathbf{m_o}_j}$.

However, we still have two problems, as follows.
\begin{enumerate}
\item The calculation of $\IG (z_j ;X_j^\mathbf{A} | X_j^{\mathbf{m_o}_j} )$ cannot be performed in a straightforward manner. 
\item The $\argmax$ operation in \aref{eq:IG_max} is a combinatorial optimization problem and incurs heavy computational cost when $\#(\mathbf{M} \setminus \mathbf{m_o}_{j})$ and $L$ become large.
\end{enumerate}
Based on some properties of the MHDP, we can obtain reasonable solutions for these two problems. 

\subsection{Monte Carlo Approximation of IG}
Equations \aref{eq:IG_max} and \aref{eq:policy} provide a robot with an appropriate criterion for selecting an action to efficiently recognize a target object. However,  at first glance, it looks difficult to calculate the IG. First, the calculation of the expectation procedure
$\mathbb{E}_{ X^{\mathbf{A} }_j | X^{\mathbf{m_o}_j}_j }[\cdot]$ requires a sum operation over all possible $X_j^\mathbf{A} $. The number of possible $X_j^\mathbf{A} $ exponentially increases when the number of elements in the BoF increases.
Second, the calculation of $P( z_j | X^{\mathbf{A}  \cup \mathbf{m_o}_j}_j )$ for each possible observation $X^\mathbf{A} _j$ requires the same computational cost as recognition in the multimodal categorization itself. Therefore, the straightforward calculation for solving (\ref{eq:policy}) is computationally impossible in a practical sense. 

However, by exploiting a characteristic property of the MHDP, an efficient Monte Carlo approximation can be derived. First, we describe $\IG$ as the expectation of a logarithm term. 

\begin{eqnarray}
\IG(z_j; X^m_j | X_j^{\mathbf{m_o}_j}) 
&=& \sum_{z_j, X^m_j}P(z_j, X^m_j | X_j^{\mathbf{m_o}_j}) \log \frac{P(z_j, X^m_j | X_j^{\mathbf{m_o}_j})}{P(z_j | X_j^{\mathbf{m_o}_j})P(X^m_j | X_j^{\mathbf{m_o}_j})}\nonumber\\
&=& \mathbb{E}_{z_j, X^m_j | X_j^{\mathbf{m_o}_j}}\big[ \log \frac{P(z_j, X^m_j | X_j^{\mathbf{m_o}_j})}{P(z_j | X_j^{\mathbf{m_o}_j})P(X^m_j | X_j^{\mathbf{m_o}_j})}\big]. \label{eq:ig_exp}
\end{eqnarray}

An analytic evaluation of (\ref{eq:ig_exp}) is also practically impossible. Therefore, we adopt a Monte Carlo method. Equation (\ref{eq:ig_exp}) suggests that an efficient Monte Carlo approximation can be performed as shown below if we can sample
\begin{equation}
(z^{[k]}_j, X^{m[k]}_j) \sim P(z_j, X^m_j|X_j^{\mathbf{m_o}_j} ),~~~ (k \in \{1, \ldots , K\}).\nonumber
\end{equation}
Fortunately, the MHDP provides a sampling procedure for  $z^{[k]}_j \sim P(z_j|X_j^{\mathbf{m_o}_j} )$ and $X^{m[k]}_j \sim P(X^m_j | z^{[k]}_j)$ in its original paper~\cite{MHDP}.
In the context of  multimodal categorization by a robot, $X^{m[k]}_j \sim P(X^m_j | z^{[k]}_j)$ is a prediction of an unobserved modality's sensation using observed modalities' sensations, i.e., cross-modal inference. The sampling process of $(z^{[k]}_j, X^{m[k]}_j) $ can be regarded as a mental simulation by a robot that predicts the unobserved modality's sensation leading to a categorization result based on the predicted sensation and observed information.
\begin{eqnarray}
(\ref{eq:ig_exp}) &\approx & \frac{1}{K}\sum_k \log \frac{P(z_j^{[k]}, X^{m[k]}_j | X_j^{\mathbf{m_o}_j})}{P(z_j^{[k]} | X_j^{\mathbf{m_o}_j})P(X^{m[k]}_j | X_j^{\mathbf{m_o}_j})}\nonumber\\
&= & \frac{1}{K}\sum_k \log \frac{P(X^{m[k]}_j | z_j^{[k]} , X_j^{\mathbf{m_o}_j})}{\underbrace{P(X^{m[k]}_j | X_j^{\mathbf{m_o}_j})}_{*}}. \label{eq:sample}
\end{eqnarray}
In (\ref{eq:sample}), $P(X^{m[k]}_j | z_j^{[k]} , X_j^{\mathbf{m_o}_j} )$ in the numerator can be easily calculated because all the parent nodes of $X^{m[k]}_j$ are given in the graphical model shown in Fig. \ref{fig:g-mhdp}. However, $P(X^{m[k]}_j|X^{\mathbf{m_o}_j}_j)$ in the denominator cannot be evaluated in a straightforward way.
Again, a Monte Carlo method can be adopted, as follows:
\begin{eqnarray}
P(X^{m[k]}_j | X_j^{\mathbf{m_o}_j}) &=& \sum_{z_j} P(X^{m[k]}_j | z_j , X_j^{\mathbf{m_o}_j})P( z_j | X_j^{\mathbf{m_o}_j})\nonumber \\
&=& \mathbb{E}_{z_j | X_j^{\mathbf{m_o}_j}}[ P(X^{m[k]}_j | z_j  , X_j^{\mathbf{m_o}_j})]\nonumber\\
 &\approx & \frac{1}{K'} \sum_{k'} P(X^{m[k]}_j|  z_j^{[k']} , X_j^{\mathbf{m_o}_j}) \label{eq:sample2}
\end{eqnarray}
where $K'$ is the number of samples for the second Monte Carlo approximation.
Fortunately, in this Monte Carlo approximation \aref{eq:sample2}, we can reuse the samples drawn in the previous Monte Carlo approximation efficiently.
By substituting (\ref{eq:sample2}) for (\ref{eq:sample}), we finally obtain the approximate IG for the criterion of active perception, i.e., our proposed method, as follows:
\begin{equation}
\IG (z_j ;X_j^m | X_j^{\mathbf{m_o}_j} ) \approx  \frac{1}{K} \sum_{k}  \log \frac{P(X^{m[k]}_j|  z_j^{[k]}, X_j^{\mathbf{m_o}_j} )}{ \frac{1}{K} \sum_{k'}    P(X^{m[k]}_j|  z_j^{[k']}, X_j^{\mathbf{m_o}_j} )}.\nonumber
\end{equation}
Note that the computational cost for evaluating $\IG$ becomes $O(K^2)$. 
In summary, a robot can approximately estimate the IG for unobserved modality information by generating virtual observations based on observed data and evaluating their likelihood.

\subsection{Sequential Decision Making as a Submodular Maximization}
If a robot wants to select $L$ actions $\mathbf{A} _j = \{a_1, a_2, \ldots , a_L\}\ (a_i \in \mathbf{M} \setminus \mathbf{m_o}_{j}) $, it has to solve 
\aref{eq:IG_max}, i.e., a combinatorial optimization problem.
The number of combinations of $L$ actions is ${}_{\#(\mathbf{M}\setminus\mathbf{m_o}_j)} C _L$, which increases dramatically when the number of possible actions $\#(\mathbf{M} \setminus \mathbf{m_o}_{j})$ and $L$ increase.
For example, Sinapov et al. (2014) gave a robot 10 different behaviors in their experiment on robotic multimodal categorization\nocite{Sinapov2014}.
Future autonomous robots will have more available actions for interacting with a target object and be able to obtain additional types of modality information through these interactions.
Hence, it is important to develop an efficient solution for the combinatorial optimization problem. 

Here again, the MHDP has advantages for solving this problem.

\begin{theorem}\label{theo:2}
The evaluation criterion for multimodal active perception $\IG (z_j ;X_j^\mathbf{A}  | X_j^{\mathbf{m_o}_j} )$ is a submodular and non-decreasing function with regard to $\mathbf{A} $.
\end{theorem}
\begin{proof}
As shown in the graphical model of the MHDP in Fig. \ref{fig:g-mhdp}, the observations for each modality $X_j^m$ are conditionally independent under the condition that a set of latent variables $z_j = \{\{k_{jt}\}_{1 \le t \le T_j},$ $ \{t_{jn}^m\}_{m \in \mathbf{M}, 1 \le n \le N^m_j}\}$%$z_j = \{\{k_{jt}\}, \{t_{jn}^m\}\}$
is given. This satisfies the conditions of the theorem by Krause et al. (2005)\nocite{Krause05}. Therefore, $\IG (z_j ;X_j^m | X_j^{\mathbf{m_o}_j} )$ is a submodular and non-decreasing function with regard to $X_j^m$.
\end{proof}

Submodularity is a property similar to the convexity of a real-valued function in a vector space.
If a set function $F : V \rightarrow R$ satisfies
\begin{equation}
F(A\cup x)- F(A) \ge F(A' \cup x) - F(A'),\nonumber
\end{equation}
where $V$ is a finite set $\forall A \subset A' \subseteq V$ and $x \notin A$, the set function $F$ has submodularity and is called a submodular function.

Function $\IG$ is not always a submodular function. However,
Krause et al. proved that $\IG(U;A)$ is submodular and non-decreasing with regard to $A\subseteq S$ if all of the elements of $S$ are conditionally independent under the condition that $U$ is given. With this theorem, Krause et al. (2005) solved the sensor allocation problem efficiently\nocite{Krause05}. Theorem \ref{theo:2} means that 
the problem \aref{eq:IG_max} is reduced to a {\em submodular maximization problem}.

It is known that the greedy algorithm is an efficient strategy for the submodular maximization problem.
Nemhauser et al. (1978) proved that the greedy algorithm can select a subset that is at most a constant factor
$(1 - 1/ \mathrm{e})$ worse than the optimal set, if the evaluation function $F(A)$ is submodular, non-decreasing, and $F(\emptyset)=0$, where $F(\cdot)$ is a set function, and $A$ is a set\nocite{Nemhauser78}. If the evaluation function is a submodular set function, a greedy algorithm is practically sufficient for selecting subsets in many cases. In sum, a greedy algorithm gives a near-optimal solution. However, the greedy algorithm is still inefficient because it requires an evaluation of all choices at each step of a sequential decision making process. 

Minoux (1978) proposed a lazy greedy algorithm to makes the greedy algorithm more efficient for the submodular evaluation function\nocite{Minoux78}. The lazy greedy algorithm can reduce the number of evaluations by using the characteristics of a submodular function.

In this paper, we propose the use of the {\it lazy greedy algorithm} for selecting $L$ actions to recognize a target object on the basis of the submodular property of $\IG$.
The final greedy and lazy greedy algorithms for MHDP-based active perception, i.e., our proposed methods, are shown in Algorithms~\ref{alg:greedy} and~\ref{alg:lazy}, respectively.

The main contribution of the lazy greedy algorithm is to reduce the computational cost of active perception. The majority of the computational cost originates from the number of times a robot evaluates $\IG_m$ for determining action sequences. When a robot has to choose $L$ actions, the brute-force algorithm that directly evaluates all alternatives $\mathbf{A} \in \mathbf{F}^{\mathbf{m_o}_j}_L$
using \aref{eq:IG_max} requires ${}_{\#(\mathbf{M}\setminus\mathbf{m_o}_j)} C _L$ evaluations of $\IG (z_j ;X_j^\mathbf{A}  | X_j^{\mathbf{m_o}_j} )$.
In contrast, the greedy algorithm requires $\{{\#(\mathbf{M}\setminus\mathbf{m_o}_j)} +  ({\#(\mathbf{M}\setminus\mathbf{m_o}_j)} -1) + \ldots + ({\#(\mathbf{M}\setminus\mathbf{m_o}_j)} -L +1)\}$ evaluations of $\IG (z_j ;X_j^m | X_j^{\mathbf{m_o}_j} )$, i.e., $O(ML)$. The lazy greedy algorithm incurs the same computational cost as the greedy algorithm only in the worst case. However, practically, the number of re-evaluations in the lazy greedy algorithm is quite small. Therefore, the computational cost of the lazy greedy algorithm increases almost in proportion to $L$, i.e., almost linearly. 
The memory requirement of the proposed method is also quite small. Both the greedy and lazy greedy algorithms only require memory for $\IG_m$ for each modality and $K$ samples for the Monte Carlo approximation. These requirements are negligibly small compared with the MHDP itself.

\begin{algorithm}[tb!p]
\caption{Greedy algorithm.}
\label{alg:greedy}
\begin{algorithmic}
\REQUIRE MHDP is trained using a training data set.
\STATE The $j$-th object is found.
\STATE $\mathbf{m_o}_j$ is initialized, and $X_j^{\mathbf{m_o}_j}$ is observed.
\FOR{$l=1$ to $L$}
\FORALL{$m \in \mathbf{M} \setminus \mathbf{m_o}_{j}$}
\FOR{$k=1$ to $K$}
\STATE Draw \[(z^{[k]}_j, X^{m[k]}_j) \sim P(z_j, X^m_j|X_j^{\mathbf{m_o}_j} )\]
\ENDFOR
\STATE \[\mathrm{IG}_m  \leftarrow \frac{1}{K} \sum_{k}  \log \frac{P(X^{m[k]}_j|  z_j^{[k]}, X_j^{\mathbf{m_o}_j})}{ \frac{1}{K} \sum_{k'}    P(X^{m[k]}_j|  z_j^{[k']}, X_j^{\mathbf{m_o}_j})} \]
\ENDFOR
\STATE \[ m^* \leftarrow \argmax_{m} \mathrm{IG}_m \]
\STATE Execute the $m^*$-th action to the $j$-th target object and obtain $X_j^{m^*}$.
\STATE $\mathbf{m_o}_j \leftarrow \mathbf{m_o}_j \cup \{m^*\} $
\ENDFOR
\end{algorithmic}
\end{algorithm}

\begin{algorithm}[tb!p]
\caption{Lazy greedy algorithm.}
\label{alg:lazy}
\begin{algorithmic}
\REQUIRE The MHDP is trained using a training data set.
\STATE The $j$-th object is found.
\STATE $\mathbf{m_o}_j$ is initialized, and $X_j^{\mathbf{m_o}_j}$ is observed.
\FORALL{$m \in \mathbf{M} \setminus \mathbf{m_o}_{j}$}
\FOR{$k=1$ to $K$}
\STATE Draw \[(z^{[k]}_j, X^{m[k]}_j) \sim P(z_j, X^m_j|X_j^{\mathbf{m_o}_j} )\]
\ENDFOR
\STATE \[\mathrm{IG}_m  \leftarrow  \frac{1}{K} \sum_{k}  \log \frac{P(X^{m[k]}_j|  z_j^{[k]}, X_j^{\mathbf{m_o}_j})}{ \frac{1}{K} \sum_{k'}    P(X^{m[k]}_j|  z_j^{[k']}, X_j^{\mathbf{m_o}_j})}\]
\ENDFOR
\STATE \[ m^* \leftarrow \argmax_{m} \mathrm{IG}_m \]
\STATE Execute the $m^*$-th action to the $j$-th target object and obtain $X_j^{m^*}$.
\STATE $\mathbf{m_o}_j \leftarrow \mathbf{m_o}_j \cup \{m^*\} $
\STATE Prepare a stack $S$ for the modality indices and initialize it.
\FORALL{$m \in \mathbf{M} \setminus \mathbf{m_o}_{j}$}
\STATE $push(S,(m,\mathrm{IG}_m))$
\ENDFOR
\FOR{$l=1$ to $L-1$}
\REPEAT
\STATE $S \leftarrow descending\_sort (S)$\ // w.r.t. $\mathrm{IG}_m$
\STATE $(m^{1}, \mathrm{IG}_{m^{1}} )\leftarrow pop(S)$ , $(m^{2}, \mathrm{IG}_{m^{2}} )\leftarrow pop(S)$
\STATE // Re-evaluate $\mathrm{IG}_{m^{1}} $ as follows.
\FOR{$k=1$ to $K$}
\STATE Draw \[(z^{[k]}_j, X^{m^{1}[k]}_j) \sim P(z_j, X^{m^{1}}_j|X_j^{\mathbf{m_o}_j} )\]
\ENDFOR
\STATE \[\mathrm{IG}_{m^{1}}  \leftarrow \frac{1}{K} \sum_{k}  \log \frac{P(X^{m^1[k]}_j|  z_j^{[k]}, X_j^{\mathbf{m_o}_j})}{ \frac{1}{K} \sum_{k'}    P(X^{m^1[k]}_j|  z_j^{[k']}, X_j^{\mathbf{m_o}_j})}\]
\STATE $push(S,(m^{2},\mathrm{IG}_{m^{2}}))$, $push(S,(m^{1},\mathrm{IG}_{m^{1}}))$
\UNTIL{$\mathrm{IG}_{m^{1}} \ge \mathrm{IG}_{m^{2}}$}
\STATE $m^* \leftarrow m^{1}  $
\STATE $pop(S)$
\STATE Execute the $m^*$-th action to the $j$-th target object and obtain $X_j^{m^*}$.
\STATE $\mathbf{m_o}_j \leftarrow \mathbf{m_o}_j \cup \{m^*\} $
\ENDFOR
\end{algorithmic}
\end{algorithm}

\section{Experiment 1: Humanoid Robot}\label{sec5}
An experiment using an upper-torso humanoid robot was conducted to verify the proposed active perception method in the real-world environment. 

\subsection{Conditions}
In this experiment, RIC-Torso, developed by the RT Corporation, was used (see Fig.~\ref{fig:robot}).
RIC-Torso is an upper-torso humanoid robot that has two robot hands.
\begin{figure}[tb!p]
\centering
\includegraphics[width=70mm]{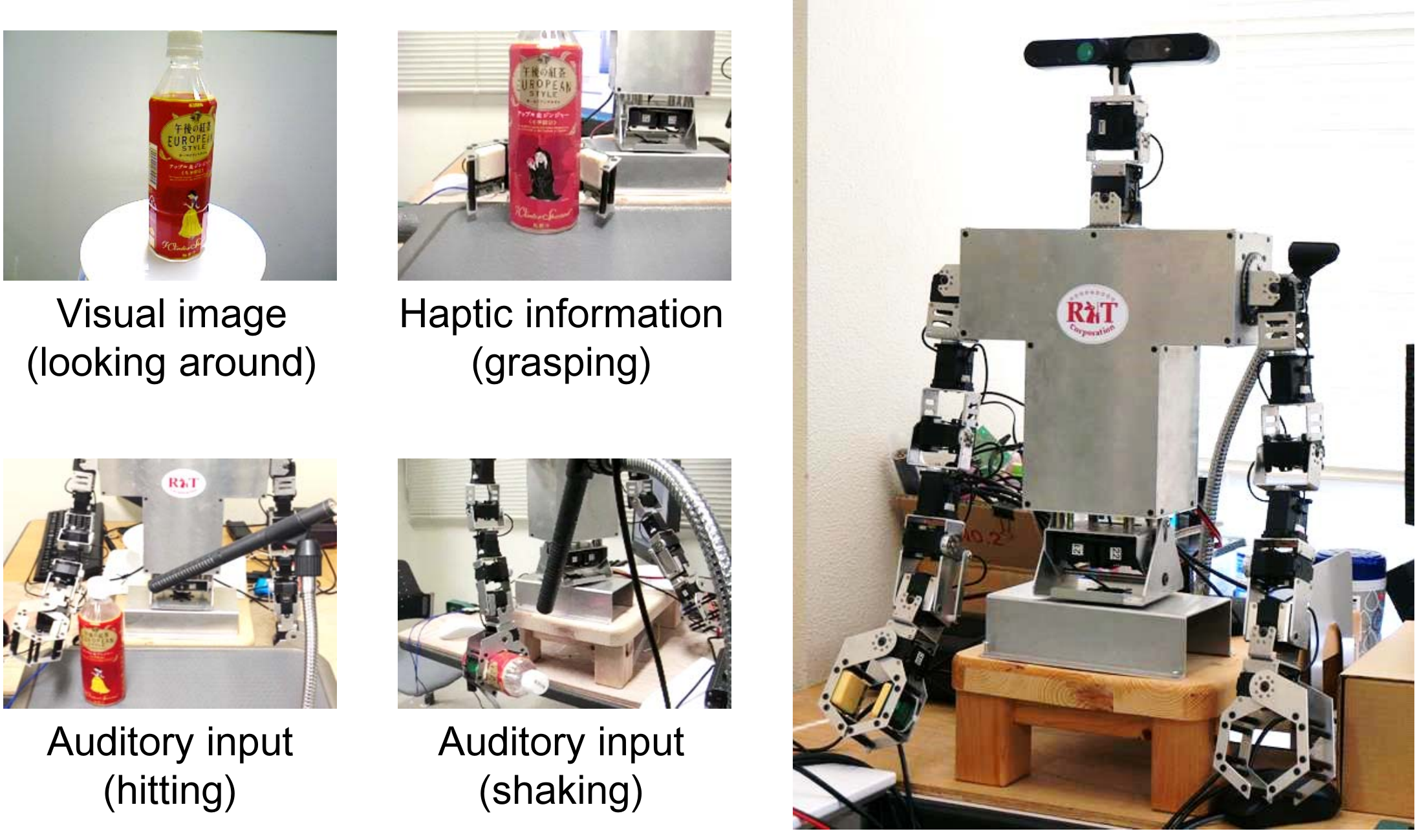}
\caption{Robot used in the experiment.}
\label{fig:robot}
\end{figure}
We prepared an experimental environment that is similar to the one in the original MHDP paper~\cite{MHDP}. 

\subsubsection{Visual information ($m^v$)}
Visual information was obtained from the Xtion PRO LIVE set on the head of the robot. The camera was regarded as the eyes of the robot.
The robot captured 74 images of a target object while it rotated on a turntable (see Fig.~\ref{fig:robot}). The size of each image was re-sized to $320 \times 240$.
Scale-invariant feature transform (SIFT) feature vectors were extracted from each captured image~\cite{SIFT}. A certain number of $128$-dimensional feature vectors were obtained from each image. {Note that the SIFT feature did not consider hue information. }
All of the obtained feature vectors were transformed into BoF representations using k-means clustering. BoF representations were used as observation data for the visual modality of the MHDP. The index for this modality was defined as $m^v$.

\subsubsection{Auditory information ($m^{as}$ and $m^{ah}$)}
Auditory information was obtained from a multipowered shotgun microphone NTG-2 by RODE Microphone. The microphone was regarded as the ear of the robot.
In this experiment, two types of auditory information were acquired. One was generated by hitting the object, and the other was generated by shaking it. The two sounds were regarded as different auditory information and hence different modality observations in the MHDP model.
The two actions, i.e., hitting and shaking, were manually programmed for the robot. When the robot began to execute an action, it also started recording the objects's sound (see Fig.~\ref{fig:robot}). The sound was recorded until two seconds after the robot finished the action. 
The recorded auditory data were temporally divided into frames, and each frame was transformed into $13$-dimensional Mel-frequency cepstral coefficients (MFCCs).  The MFCC feature vectors were transformed into BoF representations using k-means clustering in the same way as the visual information. The indices of these modalities were defined as $m^{as}$ and $m^{ah}$, respectively, for ``shake'' and ``hit.''

\subsubsection{Haptic information ($m^h$)}
Haptic information was obtained by grasping a target object using the robot's hand. 
When the robot attempted to obtain haptic information from an object placed in front of it, it moved its hand to the object and gradually closed its hand until a certain amount of counterforce was detected (see Fig.~\ref{fig:robot}).
The joint angle of the hand was measured when the hand touched the target object and when the hand stopped. The two variables and difference between the two angles were used as a three-dimensional feature vector. When obtaining haptic information, the robot grasped the target object $10$ times and obtained $10$ feature vectors. The feature vectors were transformed into BoF representations using k-means clustering in the same way as for the other information types. The index of the haptic modality was defined as $m^h$.

\subsubsection{Multimodal information as BoF representations}
In summary, a robot could obtain multimodal information from four modalities for perception. The set of modalities was $\mathbf{M} = \{m^v, m^{as}, m^{ah}, m^h \}$.
The dimensions of the BoFs were set to $25$, $25$, $25$, and $5$ for $m^v$, $m^{as}$, $m^{ah}$, and $m^h$, respectively. The dimension of each BoF corresponds to the number of clusters for k-means clustering.
The numbers of clusters, i.e., the sizes of the dictionaries, were empirically determined on the basis of a preliminary experiment on multimodal categorization. All of the training datasets were used to train the dictionaries.
The histograms of the feature vectors, i.e., the BoFs, were resampled to make their counts $N_j^{m^v}=100, N_j^{m^{as}}=80, N_j^{m^{ah}}=130$, and $N_j^{m^{h}}=30$. 
The weight of each modality $w^m$ was set to $1$. The formation of multimodal object categories itself is out of the scope of this paper. Therefore, the constants were empirically determined so that the robot could form object categories that are similar to human participants.
The number of samples $K$ in the Monte Carlo approximation for estimating IG was set to $K = 5000$.

\subsubsection{Target objects}
For the target objects, $17$ types of commodities were prepared for the experiment shown in Fig. \ref{fig:result}.
Each index on the right-hand side of the figure indicates the index of each object. The hardness of the balls, the striking sounds of the cups, and the sounds made while shaking the bottles were different depending on the object categories. Therefore, ground-truth categorization could not be achieved using visual information alone.

\subsection{Procedure}
The experimental procedure was as follows.
First, the robot formed object categories through multimodal categorization in an unsupervised manner.
An experimenter placed each object in front of the robot one by one.
The robot looked at the object to obtain visual features, grasped it to obtain haptic features, shook it to obtain auditory shaking features, and hit it to obtain the auditory striking features.
After obtaining the multimodal information of the objects as a training data set, the MHDP was trained using a Gibbs sampler. The results of multimodal categorization are shown in Fig.~\ref{fig:result}. The category that has the highest posterior probability for each object is shown in white.
These results show that the robot can form multimodal object categories using MHDP, as described in \cite{MHDP}.
After the robot had formed object categories, we fixed the latent variables for the training data set.

\begin{figure}[tb!p]
\begin{center}
\begin{tabular}{cc}
\begin{minipage}{0.49\hsize}
\centering
\includegraphics[width=80mm]{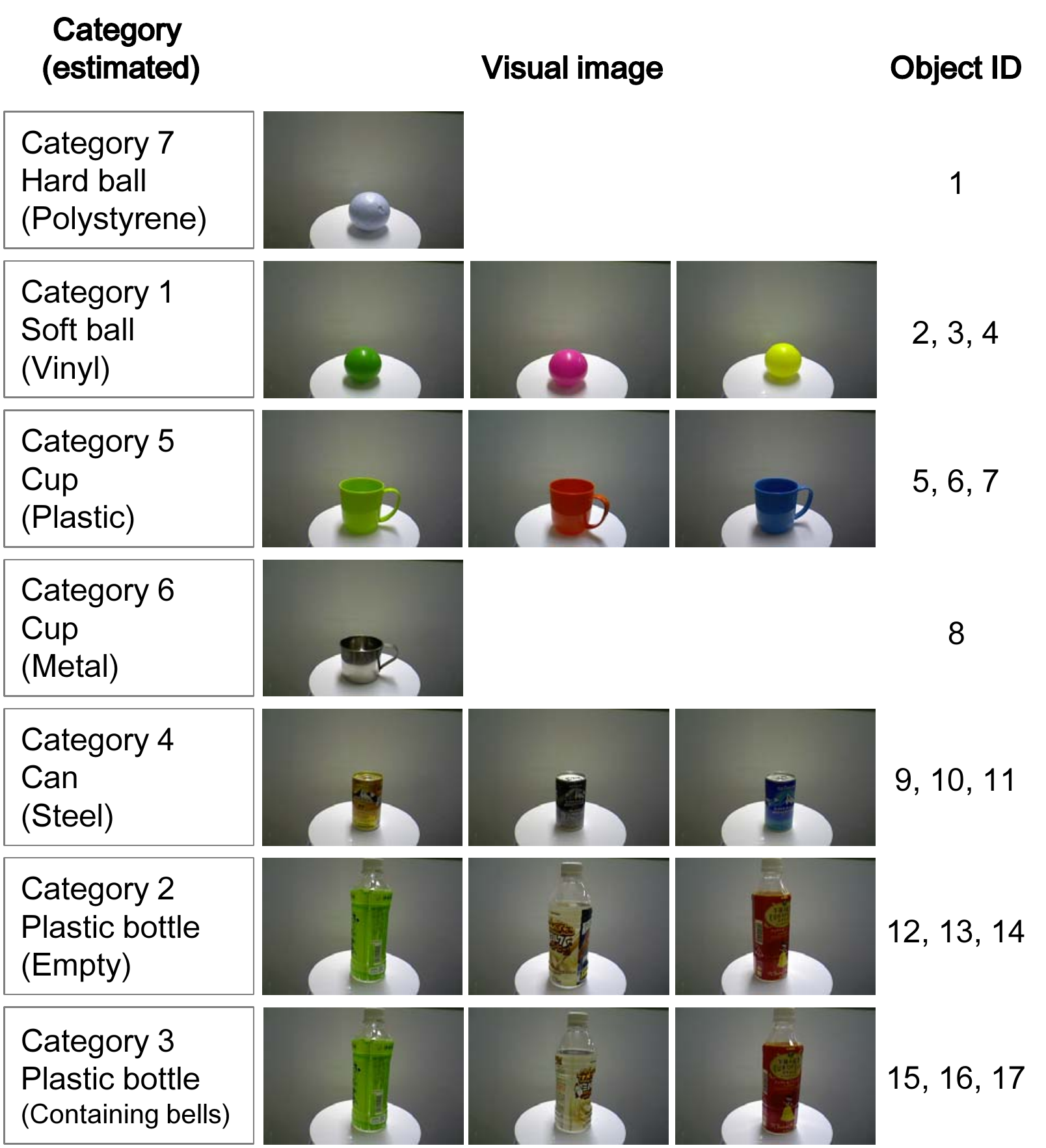}
\end{minipage}
\begin{minipage}{0.49\hsize}
\centering
\includegraphics[width=60mm]{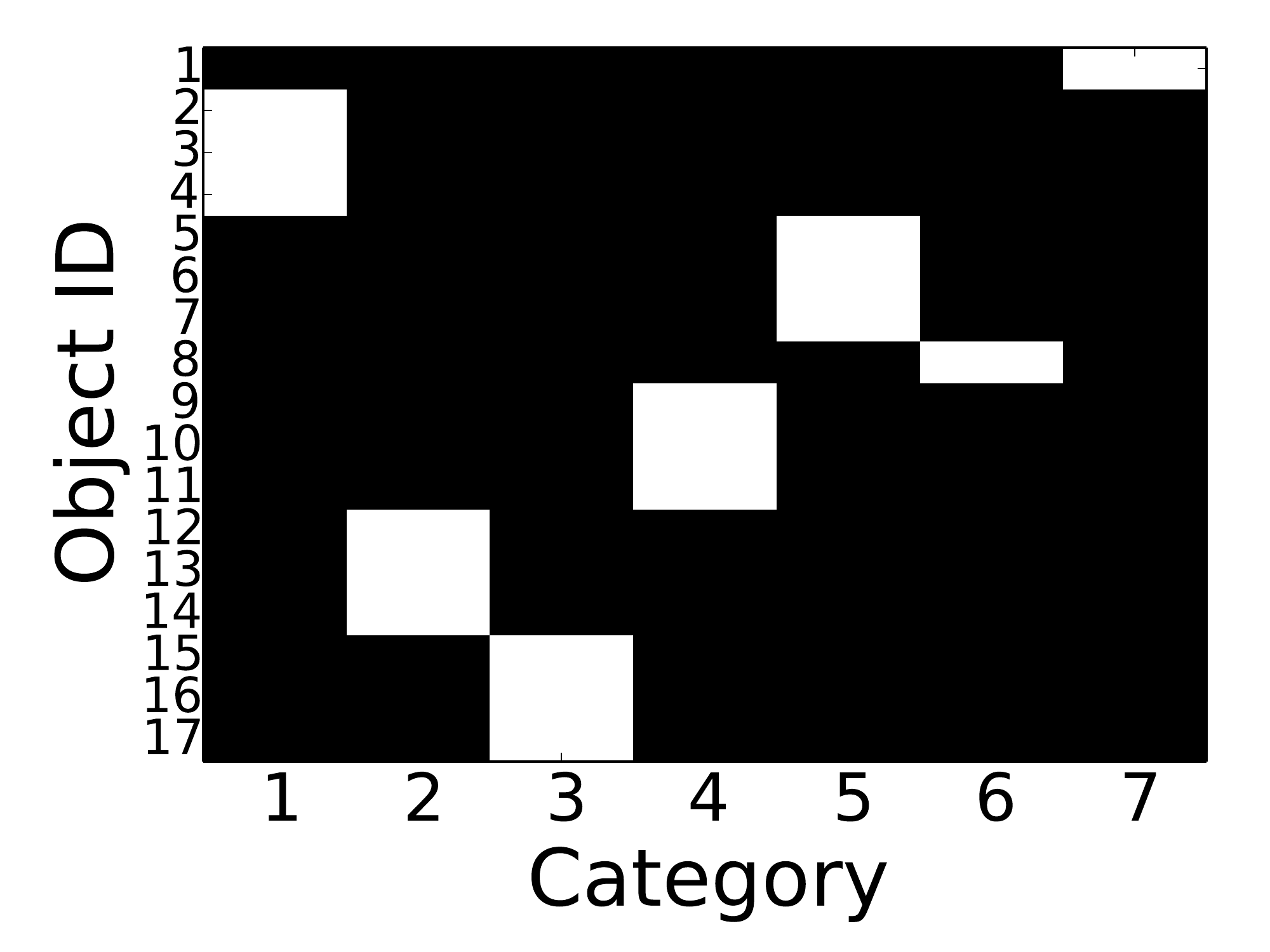}
\end{minipage}
\end{tabular}
\caption{(Left) target objects used in the experiment and (right) categorization results obtained in the experiment.}
\label{fig:result}
\end{center}
\end{figure}

Second, an experimental procedure for active perception was conducted.
An experimenter placed an object in front of the robot. The robot observed the object using its camera, obtained visual information, and set $\mathbf{m_o}_j =\{ m^v \}$. The robot then determined its next set of actions for recognizing the target object using its active perception strategy. 

\begin{figure*}[tb!p]
\includegraphics[width=\linewidth]{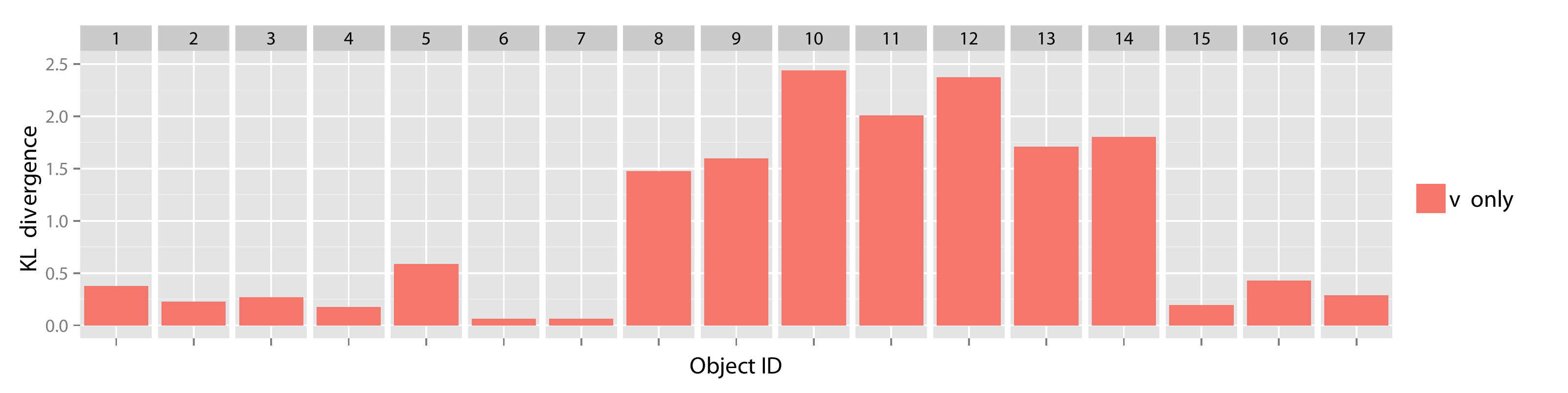}
\includegraphics[width=0.98\linewidth]{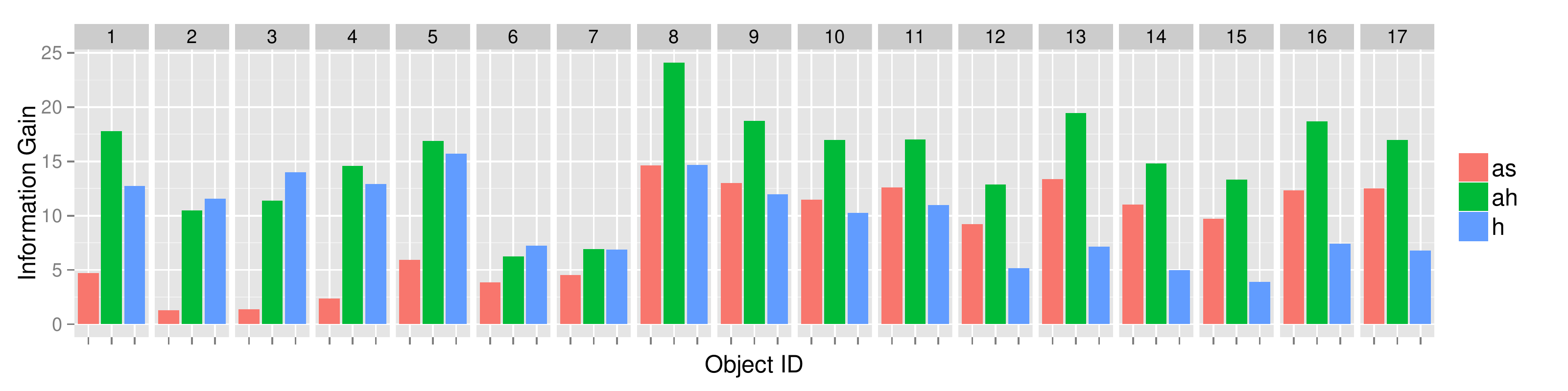}
\includegraphics[width=0.99\linewidth]{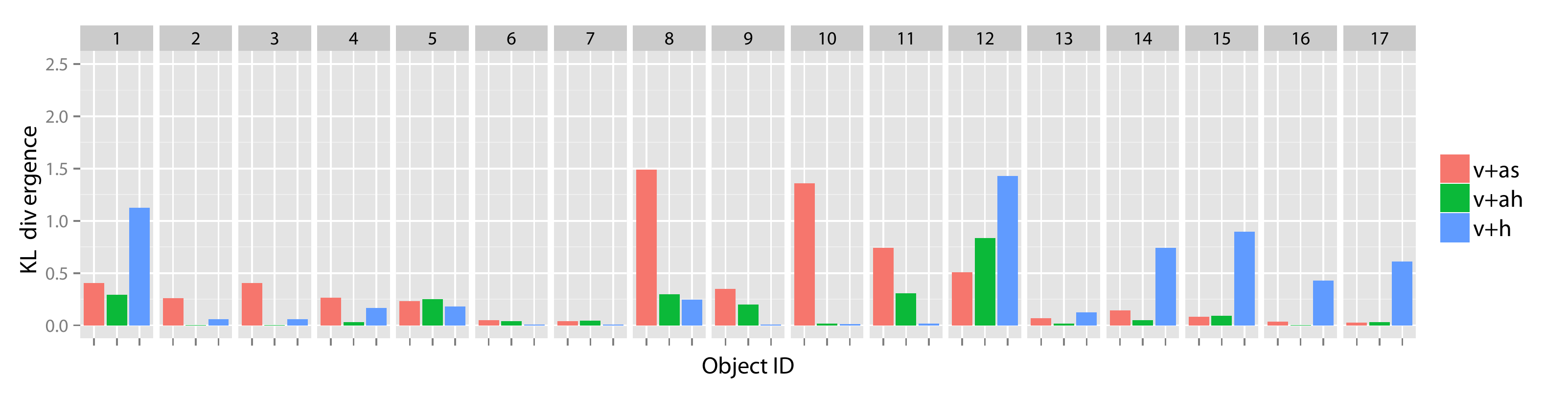}
\caption{ (Top) KL divergence between the final recognition state and the posterior probability estimated after obtaining only visual information, (middle) estimated $\IG_m$ for each object based on visual information, and (bottom) KL divergence between the final recognition state and the posterior probability estimated after obtaining only visual information and each selected action. Our theory of multimodal active perception suggests that the action with the highest information gain (shown in the middle) tends to lead its initial recognition state (whose KL divergence from the final recognition state is shown at the top) to a recognition state whose KL divergence from the final recognition state (shown at the bottom) is the smallest. These figures suggest the probabilistic relationships were satisfied as a whole.  }
\label{fig:first_kl_ig}
%\end{center}
\end{figure*}

\subsection{Results}
\subsubsection{Selecting the next action}
First, we describe results for the first single action selection after obtaining visual information.
In this experiment, the robot had three choices for its next action, i.e., $m^{as}$, $m^{ah}$, and $m^h$.
To evaluate the results of active perception, we used $\KL \big(  P(k|X_j^{\mathbf{M}}), P(k|X_j^{\mathbf{A} \cup \mathbf{m_o}_j  }) \big)${, i.e., the distance between the posterior distribution over the object categories $k$ in the final recognition state and that in the next recognition state} as an evaluation criterion on behalf of $ \KL \big( P( z_j | X_j^{\mathbf{M}} ) , P( z_j | X_j^{\mathbf{A} \cup \mathbf{m_o}_j} ) \big)$. This is the original evaluation criterion in \aref{eq:goal_KL_det} because the computational cost for evaluating $\KL \big( P( z_j | X_j^{\mathbf{M}} ) , P( z_j | X_j^{\mathbf{A} \cup \mathbf{m_o}_j} ) \big)$ is too high to calculate.

Fig.~\ref{fig:first_kl_ig} (top) shows the KL divergence between the posterior probabilities of the category after obtaining the information from all modalities and after obtaining only visual information.
With regard to some objects, e.g., objects 6 and 7, the figure shows that visual information is sufficient for the robot to recognize the objects. However, with regard to many objects, visual information alone could not lead the recognition state to the final state. However, it could be  reached using the information of all modalities.
Fig.~\ref{fig:first_kl_ig} (middle) shows $\IG_m$ calculated using the visual information for each action. Fig.~\ref{fig:first_kl_ig} (bottom) shows the KL divergence between the final recognition state and the posterior probability estimated after obtaining visual information and the information of each selected action. We observe that an action with a higher value of $\IG_m$ tended to further reduce the KL divergence, as Theorem~\ref{theo:1} suggests.
Fig.~\ref{fig:KL-sorted}
shows the average KL divergence for the final recognition state after executing an action selected by the $\IG_m$ criterion.
Actions $\IG.\text{min}$, $\IG.\text{mid}$, and $\IG.\text{max}$ denote actions that have the minimum, middle, and maximum values of $\IG_m$, respectively. These results show that $\IG.\text{max}$ clearly reduced the uncertainty of the target objects. 

The precision of category recognition after an action execution is summarized in Table~\ref{tbl:precision}. Basically, a category recognition result is obtained as the posterior distribution (\ref{eq:p}) in the MHDP. The category with the highest posterior probability is considered to be the recognition result for illustrative purposes in Table~\ref{tbl:precision}.
Obtaining information by executing $\IG.\text{max}$ almost always increased recognition performance.

\begin{table}[bt]
\caption{Number of Successfully Recognized Objects}
\label{tbl:precision}
\begin{center}
\begin{tabular}{|c||c|c|c||c|}
\hline
v only & v+IG.min & v+IG.mid & v+IG.max & Full information\\
\hline
$8/17$ & $11/17$ & $15/17$ & $ {\mathbf 16/17} $ & $17/17$ \\
\hline
\end{tabular}
\end{center}
\end{table}

\begin{figure}[tb!p]
\begin{center}
\includegraphics[width=0.6\linewidth]{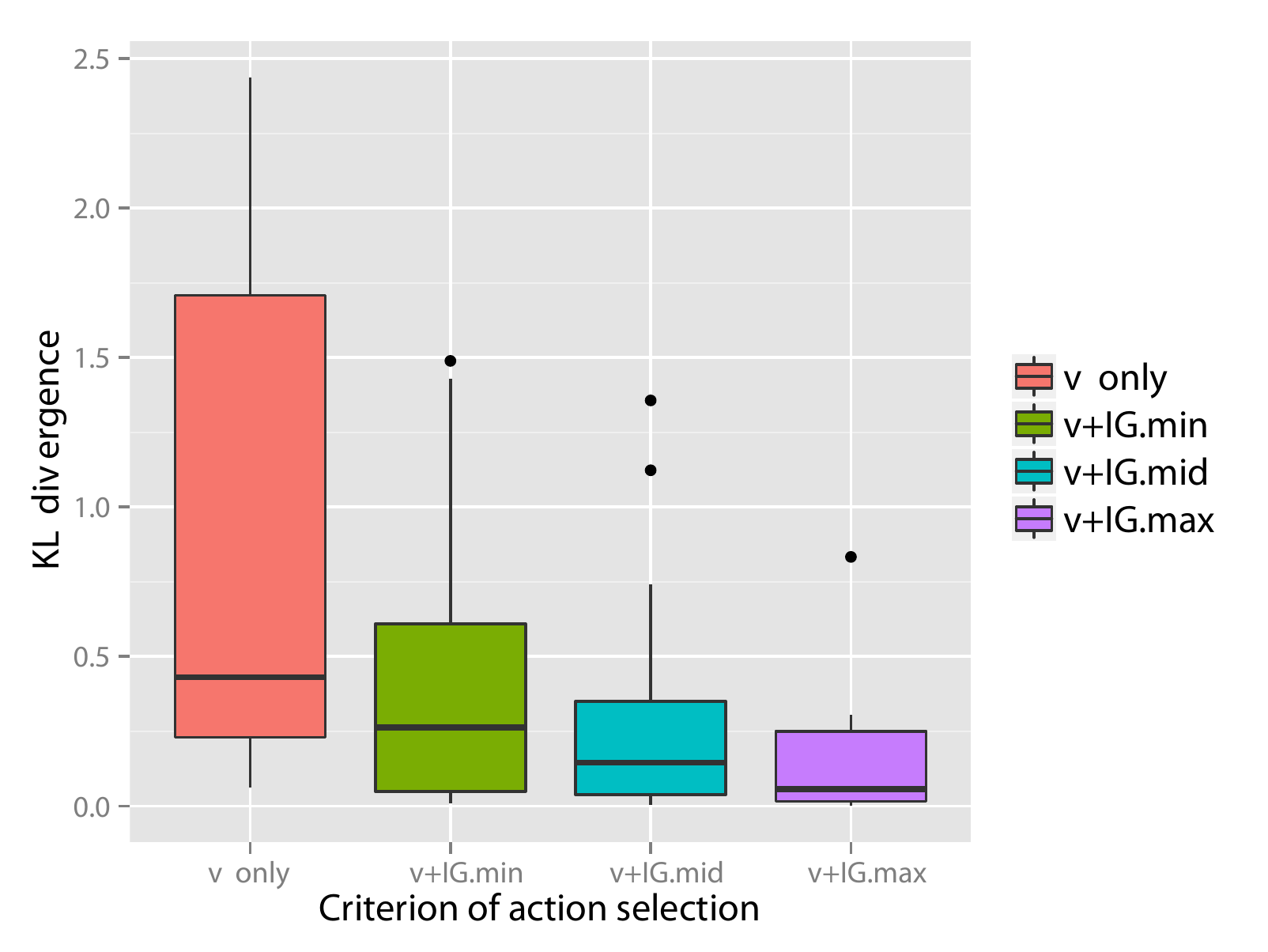}
\caption{Reduction in the KL divergence by executing an action selected on the basis of the $\IG_m$ maximization criterion. The KL divergences between the recognition state after executing the second action and the final recognition state are calculated for all objects and shown with box plot. This shows that an action with more information brings the recognition of its state closer to the final recognition state.}
\label{fig:KL-sorted}
\end{center}
\end{figure}

Examples of changes in the posterior distribution are shown in Figs.~\ref{fig:actionresult8} and \ref{fig:actionresult12} for objects 8 (``metal cup'') and 12 (``plastic bottle containing bells''), respectively.
The robot could not clearly recognize the category of object 8 after obtaining visual information. Action $\IG_m$ in Fig.~\ref{fig:first_kl_ig} shows that $m^{ah}$ was $\IG.\text{max}$ for the 8th object. Fig.~\ref{fig:actionresult8} shows that $m^{ah}$ reduced the uncertainty and allowed the robot to correctly recognize the object, as evidenced by category 6, a metal cup.
This means that the robot noticed that the target object was a metal cup by hitting it and listening to its metallic sound.
The metal cup did not make a sound when the robot shook it. Therefore, the $\IG$ for $m^{as}$ was small. 
As Fig.~\ref{fig:actionresult12} shows, the robot first recognized the 12th object as a plastic bottle containing bells with high probability and as an empty plastic bottle with a low probability. Fig.~\ref{fig:first_kl_ig} shows that the $\IG_m$ criterion suggested $m^{ah}$ as the first alternative and $m^{as}$ as the second alternative.  Fig.~\ref{fig:actionresult12} shows that $m^{as}$ and $m^{ah}$ could determine that the target object was an empty plastic bottle, but $m^h$ could not.

As humans, we would expect to differentiate an empty bottle from a bottle containing bells by shaking or hitting the bottle, and differentiate a metal cup from a plastic cup by hitting it.
The proposed active perception method constructively reproduced this behavior in a robotic system using an unsupervised multimodal machine learning approach.

\begin{figure}[tb!p]
\begin{center}
\begin{tabular}{cc}
\begin{minipage}{0.49\hsize}
\centering
\includegraphics[width=\linewidth]{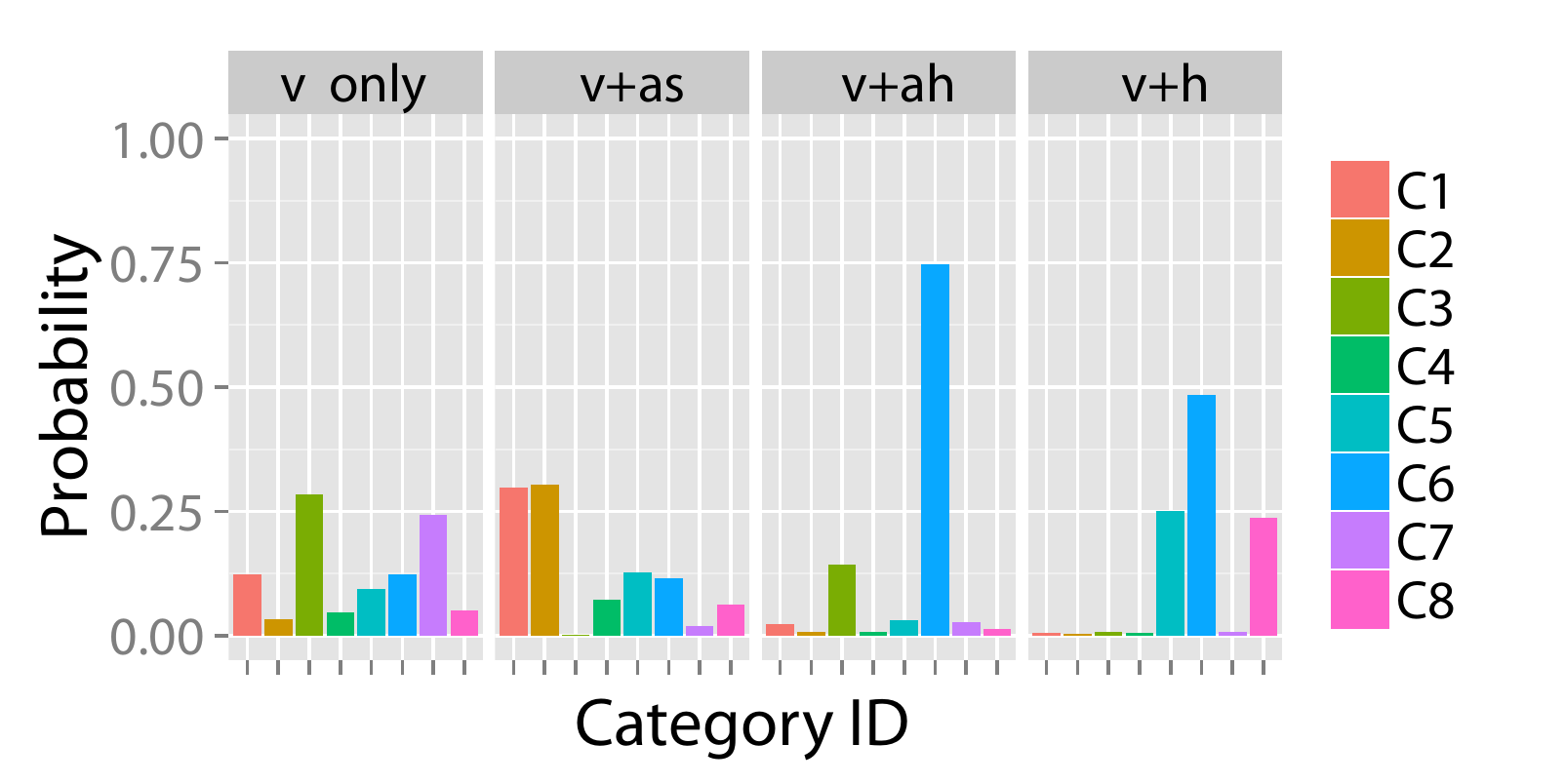}
\caption{Posterior probability of the category for object 8 after executing each action. {These results show that the action with the highest information gain, i.e., $ah$, allowed the robot to efficiently estimate that the true object category was ``metal cup.''}}\label{fig:actionresult8}
\end{minipage}
\begin{minipage}{0.49\hsize}
\centering
\includegraphics[width=\linewidth]{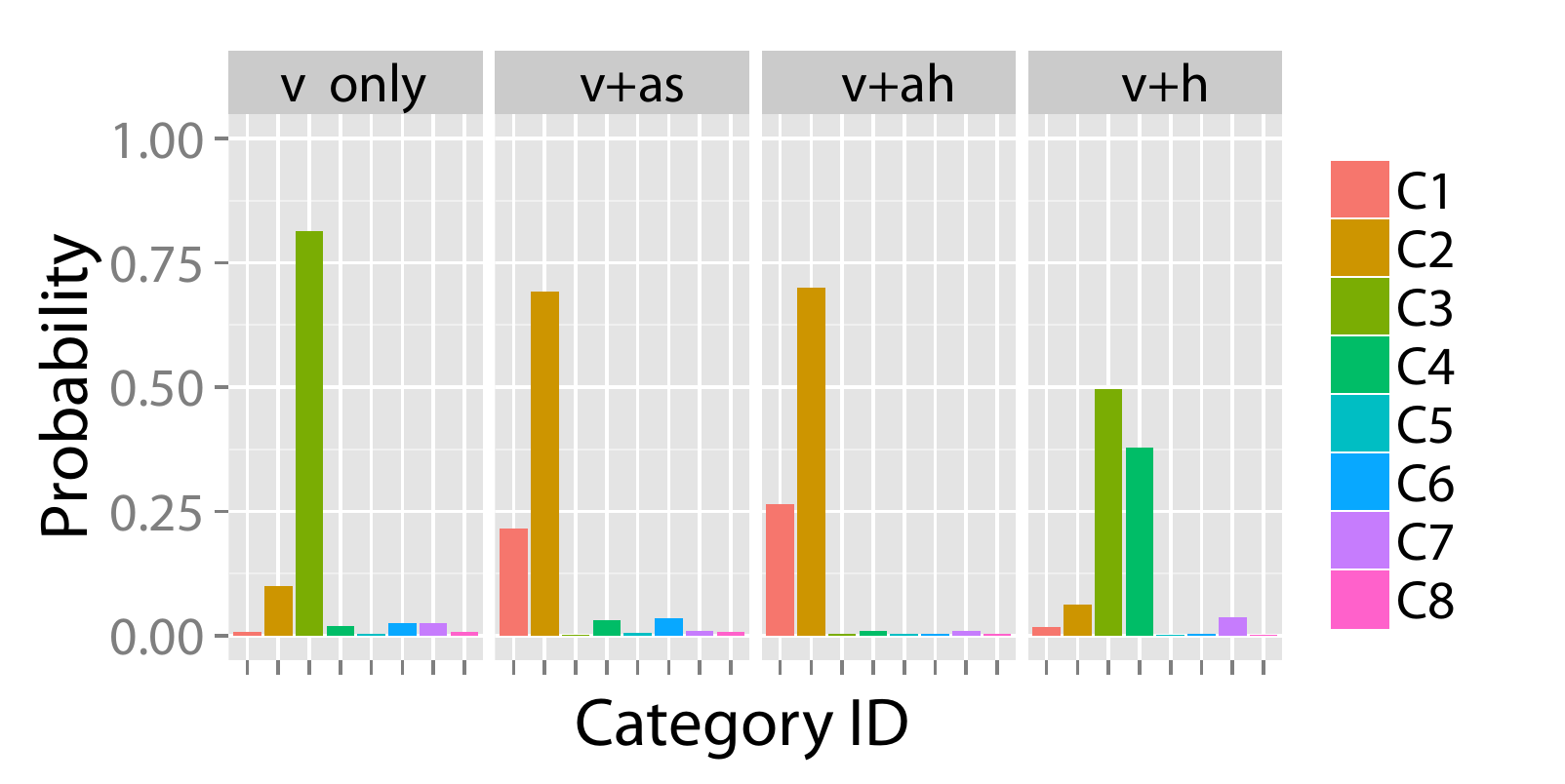}
\caption{Posterior probability of the category for object 12 after executing each action. {These results show that the actions with the highest and second highest information gain, i.e., $ah$ and $as$, allowed the robot to efficiently estimate that the true object category was ``plastic bottle containing bells.''}}\label{fig:actionresult12}
\end{minipage}
\end{tabular}
\end{center}
\end{figure}

\subsubsection{Selecting the next set of multiple actions}
We evaluated the greedy and lazy greedy algorithms for active perception sequential decision making. The KL divergence from the final state for all target objects is averaged at each step and shown in Fig.~\ref{fig:sequence}. 
For each condition, the KL divergence gradually decreased and reached almost zero. However, the rate of decrease notably differed. As the theory of submodular optimization suggests, the greedy algorithm was shown to be a better solution on average and slightly worse than the best case~\cite{Nemhauser78}. The best and worst cases were selected after all types of sequential actions had been performed. 
The ``average'' is the average of the KL divergence obtained by all possible types of sequential actions.  
The results for the lazy greedy algorithm were almost the same as those of the greedy algorithm, as Minoux et al. (1978) suggested\nocite{Minoux78}.

\begin{figure}[tb!p]
\begin{center}
\includegraphics[width=0.7\linewidth]{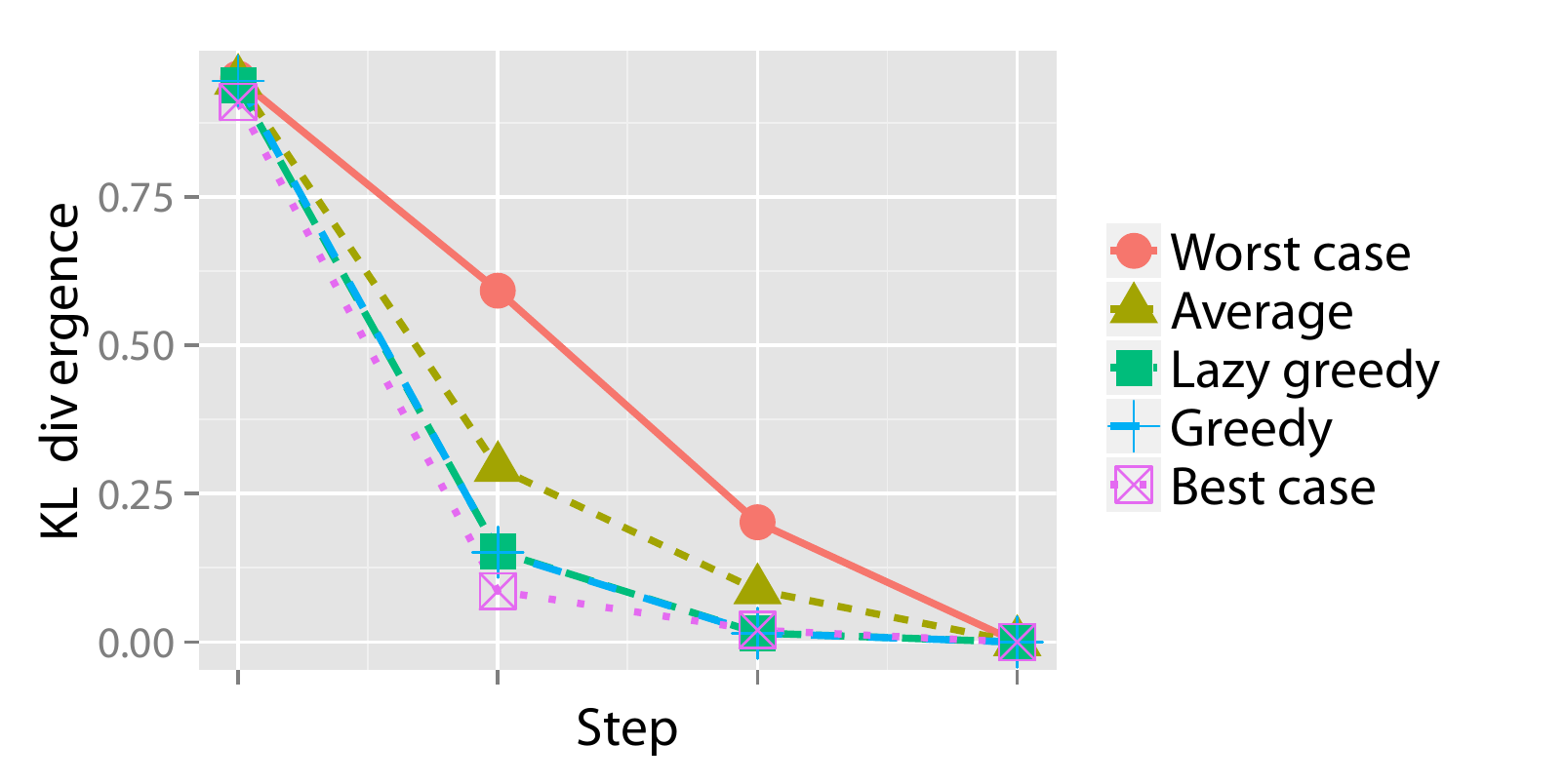}
\caption{KL divergence from the final state at each step for each sequential action selection procedure. Note that the line of the lazy greedy algorithm is overlapped by that of the greedy algorithm.}
\label{fig:sequence}
\end{center}
\end{figure}
	
The sequential behaviors of $\IG_m$ were observed to determine if their behaviors were consistent with our theories. 
For example, the changes in $\IG_m$ at each step as the robot sequentially selected its action to perform on object 10 using the greedy algorithm is shown in Fig.~\ref{fig:IG10}.
Theorem~\ref{theo:2} shows that the IG is a submodular function. This predicts that $\IG_m$ decreases monotonically when a new action is executed in active perception. When the robot obtained only visual information (v only in Fig.~\ref{fig:IG10}), all values of $\IG_m$ were still large. After $m^{ah}$ was executed on the basis of the greedy algorithm, $\IG_{m^{ah}}$ became zero. At the same time, $\IG_{m^{as}}$ and $\IG_{m^{h}}$ decreased. In the same way, all values of $\IG_{m}$ gradually decreased monotonically.

\begin{figure}[tb!p]
\begin{center}
\begin{tabular}{cc}
\begin{minipage}{0.49\hsize}
\centering
\includegraphics[width=\linewidth]{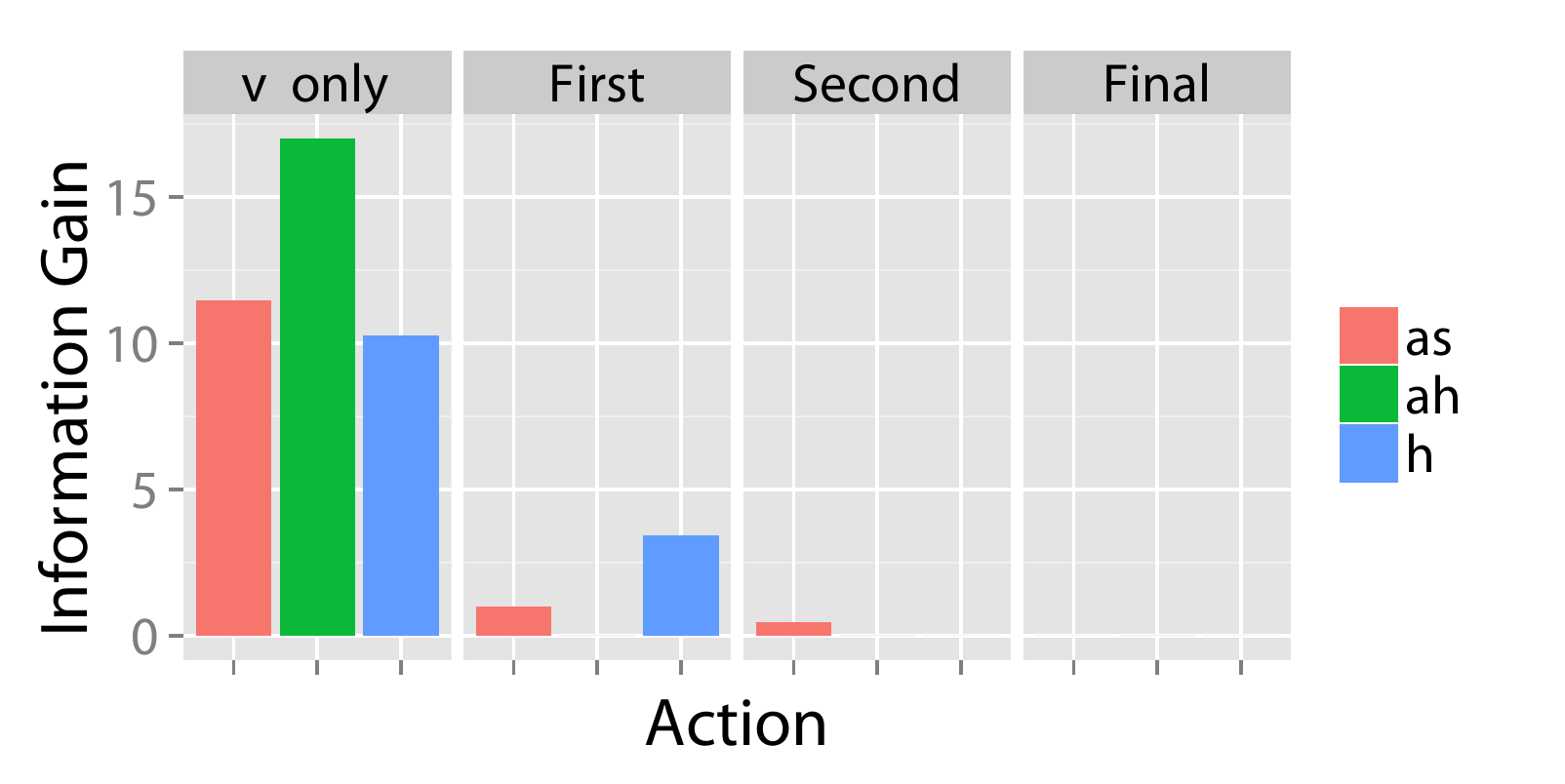}
\caption{$\IG_m$ at each step for object 10 when the greedy algorithm is used.}
\label{fig:IG10}
\end{minipage}
\begin{minipage}{0.49\hsize}
\centering
\includegraphics[width=\linewidth]{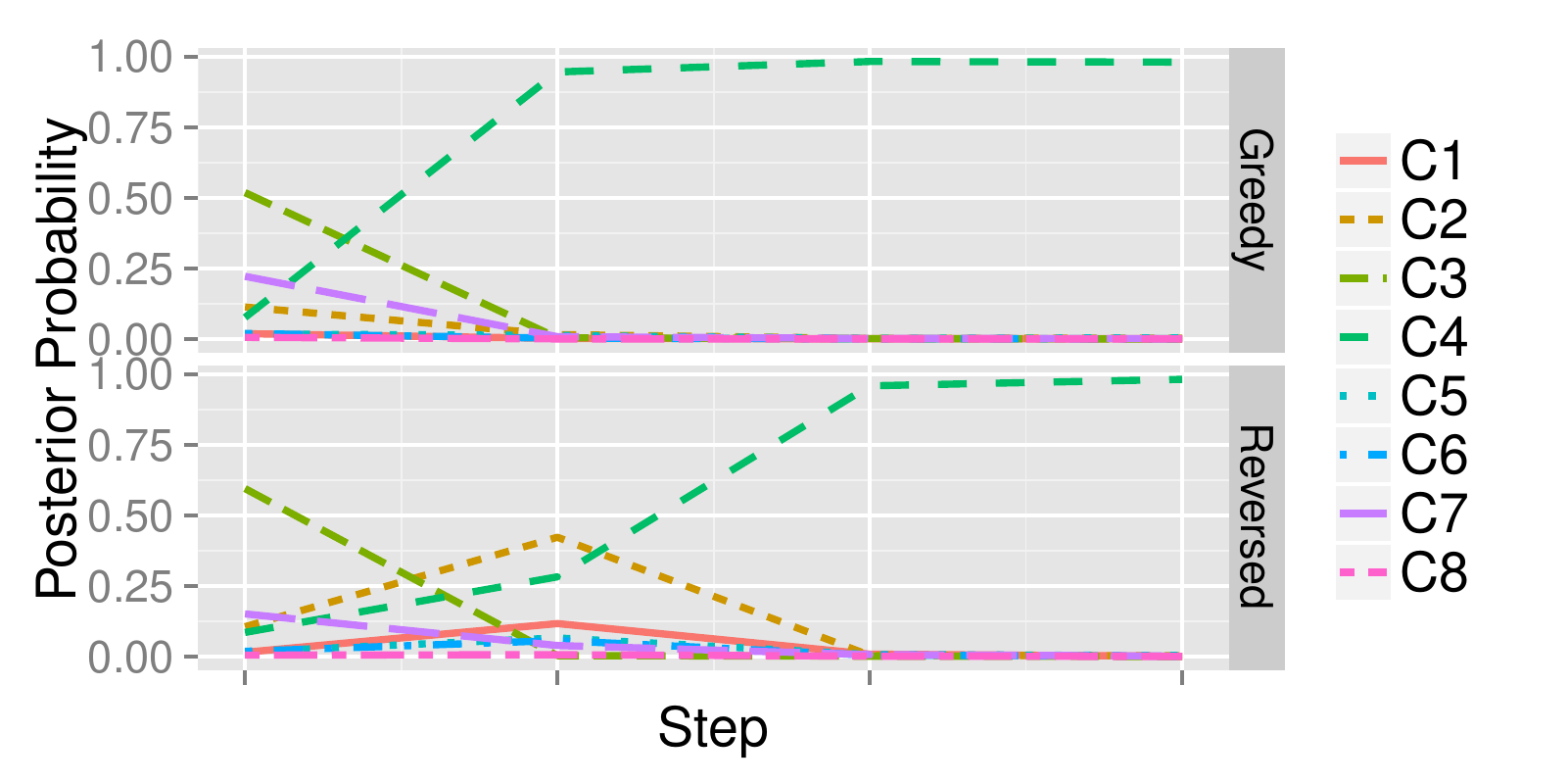}
\caption{Time series of the posterior probability of the category for object 10 during sequential action selection based on (top) the greedy algorithm, i.e., $m^{ah} \rightarrow m^{h} \rightarrow m^{as}$, and (bottom) its reverse order , i.e., $m^{as} \rightarrow m^{h} \rightarrow m^{ah}$. }
\label{fig:posterior10}
\end{minipage}
\end{tabular}
\end{center}
\end{figure}

Fig.~\ref{fig:posterior10} shows the time series of the posterior probability of the category for object 10 during sequential active perception.  Using only visual information, the robot misclassified the target object as a plastic bottle containing bells (category 3).
The action sequence in reverse order did not allow the robot to recognize the object as a steel can at its first step and change its recognition state to an empty plastic bottle (category 4). After the second action, i.e., grasping ($m^h$), the robot recognized the object as a steel can.
In contrast, the greedy algorithm could determine that the target object was in category 4, i.e., steel can, with its first action.

The effect of the number of samples $K$ for the Monte Carlo approximation was observed. Fig.~\ref{fig:IG_sd} shows the relation between $K$ and the standard deviation of the estimated $\IG_m$ for the 15th object for each action after obtaining a visual image. This figure shows that estimation error gradually decreases when $K$ increases. Roughly speaking, $K \ge 1000$ seems to be required for an appropriate estimate of $\IG_m$ in our experimental setting. Evaluation of $\IG_m$ required less than 1 second, which is far shorter than the time required for action execution by a robot. This means that our method can be used in a real-time manner.
\begin{figure}[tb!p]
\begin{center}
\includegraphics[width=0.49\linewidth]{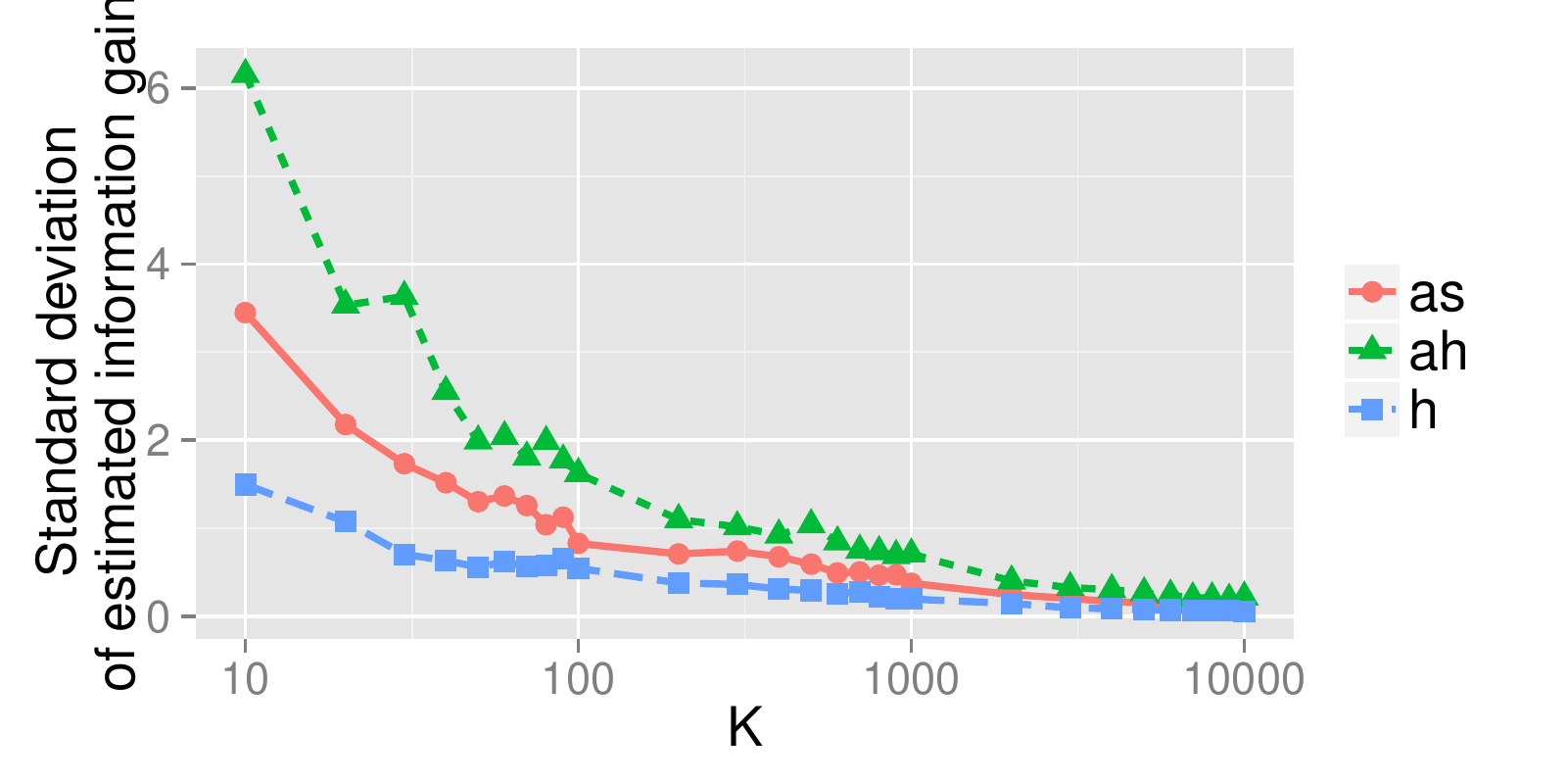}
\caption{Standard deviation of the estimated information gain $\IG_m$ for the 15th object. For each $K$, 100 values of the estimated information gain $\IG_m$ were obtained, and their standard deviation is shown.}
\label{fig:IG_sd}
\end{center}
\end{figure}

These empirical results show that the proposed method for active perception allowed a robot to select appropriate actions sequentially to recognize an object  in the real-world environment and in a real-time manner. 
It was shown that the theoretical results were supported, even in the real-world environment.

\section{Experiment 2: Synthetic Data}\label{sec6}
In experiment 1, the numbers of classes, actions, and modalities as well as the size of dataset were limited. In addition, it was difficult to control the experimental settings so as to check some interesting theoretical properties of our proposed method. 
Therefore, we performed a supplemental experiment, Experiment 2, using synthetic data comprising 21 object types, 63 objects, and 20 actions, i.e., modalities. 

First, we checked the validity of our active perception method when the number of types of actions increases. 
Second, we checked how the method worked when two classes were assigned to the same object. Although the MHDP can categorize an object into two or more categories in a probabilistic manner, each object was classified into a single category in the previous experiment. 

\subsection{Conditions}
A synthetic dataset was generated using the generative model that the MHDP assumes (see Fig.~\ref{fig:g-mhdp}).
We prepared 21 virtual object classes, and three objects were generated from each object class, i.e., we obtained 63 objects in total. Among the object classes, 14 object classes are ``pure,'' and seven object classes are ``mixed.''  For each pure object class, a multinomial distribution was drawn from the Dirichlet distribution corresponding to each modality.
We set the number of modalities $M=20$.
The hyperparameters of the Dirichlet distributions of the modalities were set to $\alpha_0^m = 0.4(m-1)$ for $m > 1$. For $m = 1$, we set $\alpha_0^1 = 10$. For each mixed object class, a multinomial distribution for each modality was prepared by mixing the distributions of the two pure object classes. Specifically, the multinomial distribution for the $i$-th mixed object was obtained by averaging those of the $(2i-1)$-th and the $2i$-th object classes.
The observations for each modality of each object were drawn from the multinomial distributions corresponding to the object's class. The count of the BoFs for each modality was set to 20. Finally, 42 pure virtual objects and 21 mixed virtual objects were generated.

The experiment was performed almost in the same way as experiment 1. First, multimodal categorization was performed for the 63 virtual objects, and 14 categories were successfully formed in an unsupervised manner. The posterior distributions over the object categories are shown in Fig.~\ref{fig:exp2colormap}.
Generally speaking, mixed objects were categorized into two or more classes. After categorization, a virtual robot was asked to recognize all of the target objects using the proposed active perception method. 

\begin{figure}[tb!p]
\begin{center}
\includegraphics[width=0.49\linewidth]{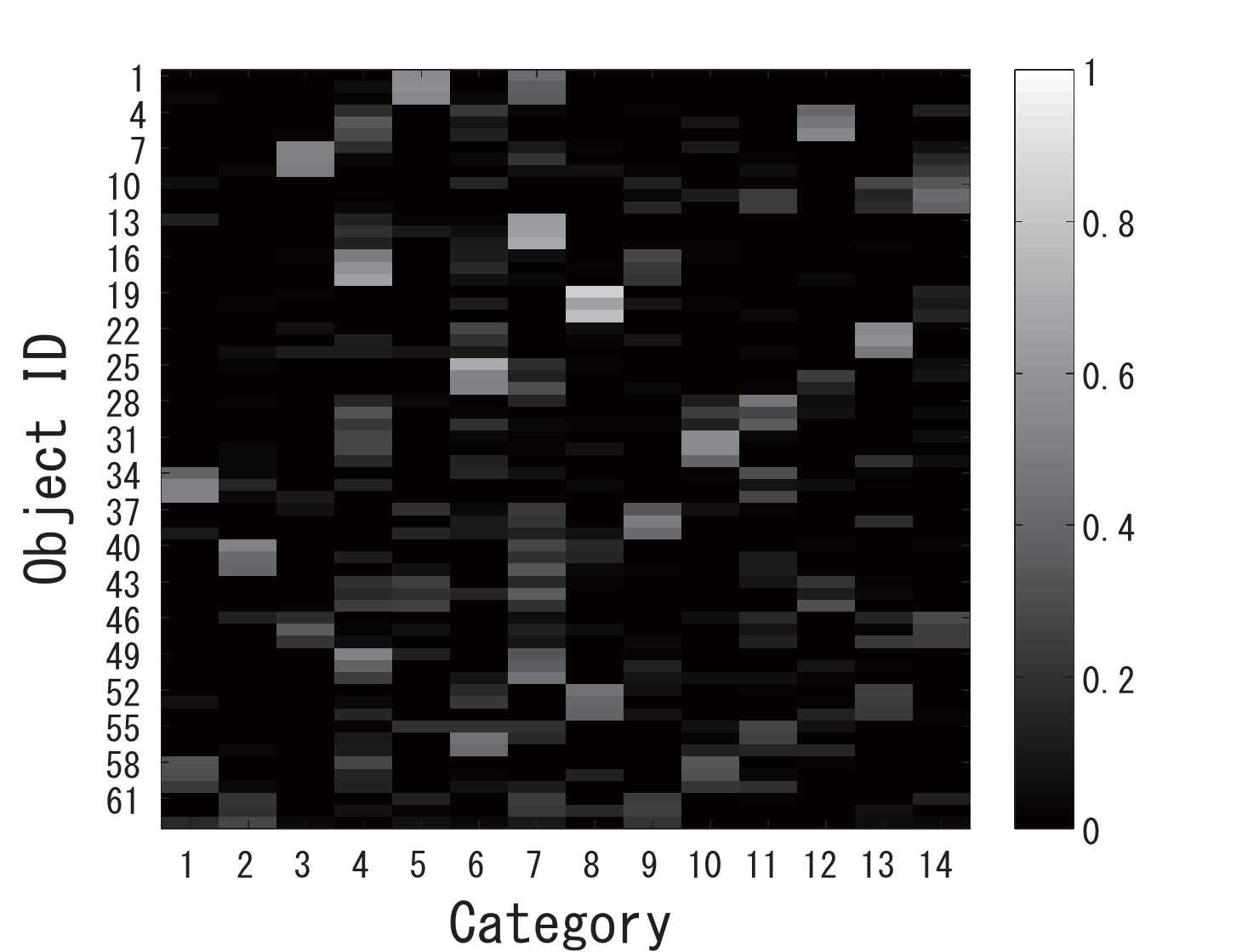}
\caption{Categorization results for the posterior probability distributions for each object.}
\label{fig:exp2colormap}
\end{center}
\end{figure}

\subsection{Results}
We compared the greedy, lazy greedy, and random algorithms for the active perception sequential decision making process. The random algorithm is a baseline method that determines the next action randomly from the remaining actions that have not been taken. In other words, the random algorithm is the case in which a robot does not employ any active perception algorithms.

The KL divergence from the final state for all target objects is averaged at each step and shown in Fig.~\ref{fig:flow_syn}. For each condition, the KL divergence gradually decreased and reached almost zero. However, the rate of decrease was different. The greedy and lazy greedy algorithms were clearly shown to be better solutions on average than the random algorithm. In contrast with experiment 1, the best and worst cases could not practically be calculated because of the prohibitive computational cost. Interestingly, the lazy greedy algorithm has almost the same performance as the greedy algorithm, as the theory suggests,
although the laziness reduced the computational cost in reality. 

\begin{figure}[tb!p]
\begin{center}
\includegraphics[width=0.7\linewidth]{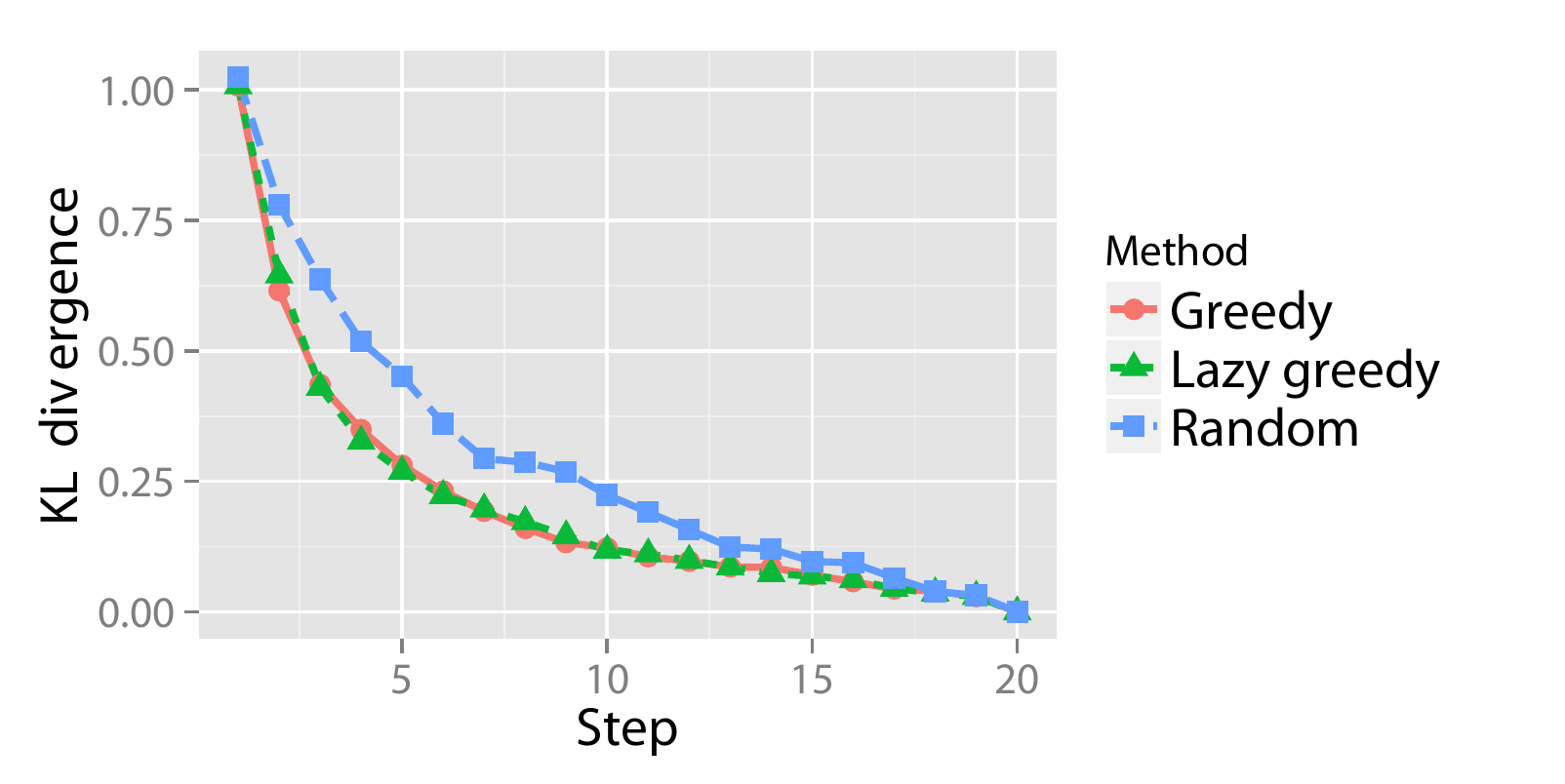}
\caption{KL divergence from the final state at each step for each sequential action selection procedure. }
\label{fig:flow_syn}
\end{center}
\end{figure}

The number of times the robot evaluated $\IG_m$ to determine the action sequences for all executable counts of actions $ L = 1, 2, \ldots , M$ is summarized for each method.
The number of times the lazy greedy algorithm was required for each target object was $71.7$ $(SD = 5.2)$ on average, and that of the greedy algorithm was $190$.
Theoretically, the greedy and lazy greedy algorithms require $O(M^2)$ evaluations.
Practically, the number of re-evaluations needed by the lazy greedy algorithm is quite small.
In contrast, the brute-force algorithm requires $O(2^M)$ evaluations, i.e., far more evaluations of $\IG$ are required. 

Next, a case in which two classes were assigned to the same object was investigated. The target dataset contained ``mixed'' objects. The results also imply that our method works well even when two classes are assigned to the same object. This is because our theory is completely derived on the basis of the probabilistic generative model, i.e., the MHDP. We show a typical result. 

Fig.~\ref{fig:posterior51} shows the time series of the posterior probability of the category for object 51, i.e., one of the mixed objects, during sequential active perception. This shows that the greedy and lazy greedy algorithms quickly categorized the target object into two categories ``correctly.'' Our formulation assumes the categorization result to be a posterior distribution. Therefore, this type of probabilistic case can be treated naturally.

\begin{figure}[tb!p]
\begin{center}
\includegraphics[width=0.7\linewidth]{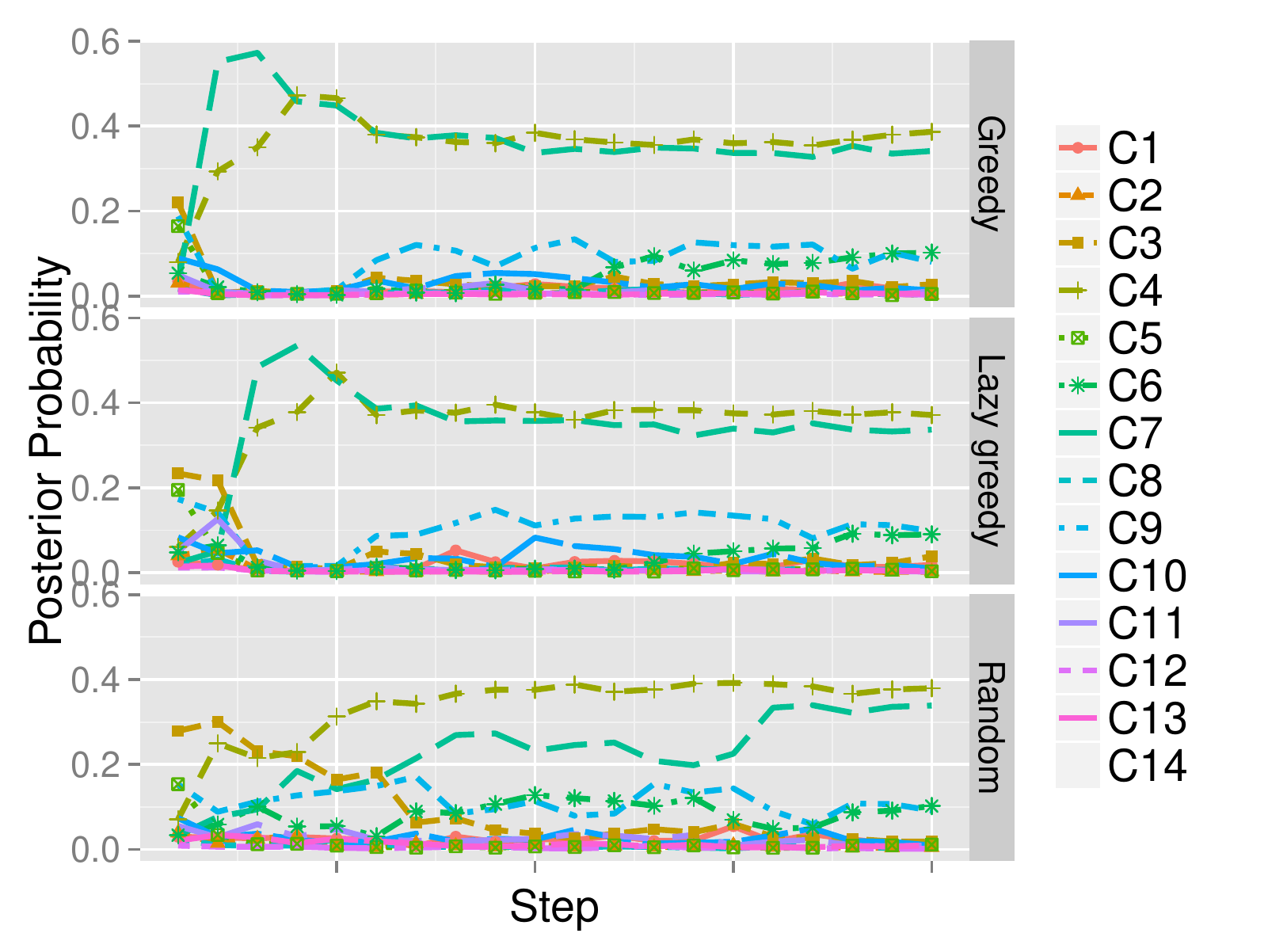}
\caption{Time series of the posterior probability of the category for object 51 during sequential action selection based on (top) the greedy algorithm, (middle) the lazy greedy algorithm, and (bottom) the random selection procedure. }
\label{fig:posterior51}
\end{center}
\end{figure}

\section{Conclusion}\label{sec7}
In this paper, we described an MHDP-based active perception method for robotic multimodal object category recognition. We formulated a new active perception method on the basis of the MHDP~\cite{MHDP} . 

First, we proposed an action selection method based on the IG criterion and proved that
 IG is an optimal criterion for active perception from the viewpoint of reducing the expected KL divergence between the final and current recognition states. Second, we derived a Monte Carlo approximation method for evaluating $\IG$ efficiently and made the action selection method executable. Third, we proved that the IG has a submodular property and reduced the sequential active perception problem to a submodular maximization problem. 
 Given the theoretical results, we proposed to use the lazy greedy algorithm for selecting a set of actions for active perception. 
 It is important to note that all of the three theoretical contributions mentioned above were naturally derived from the characteristics of the MHDP. These contributions are clearly a result of the theoretical soundness of the MHDP. In this sense, our theorems reveal a new advantage of the MHDP that other several heuristic multimodal object categorization methods do not have.
 
 To evaluate the proposed methods empirically, we conducted experiments using an upper-torso humanoid robot and a synthetic dataset. Our results showed that the method enables the robot to actively select actions and recognize target objects quickly and accurately.

One of the most interesting points of this paper is that not only object categories but also an action selection policy for object recognition can be formed in an unsupervised manner.
From the viewpoint of cognitive developmental robotics, providing an unsupervised learning model for bridging the development between perceptual and action systems is meaningful for shedding a new light on the computational understanding of cognitive development~\cite{Cangelosi2015,Asada2009a}. 
It is believed that the coupling of action and perception is important for an embodied cognitive system~\cite{Pfeifer2001}. 

The advantage of this paper compared with the related works % in Section~\ref{sec2} 
is that our action selection method for multimodal category recognition has a clear theoretical basis and is tightly connected to the computational model for multimodal object categorization, i.e., MHDP.
This fact gives our active perception method a theoretical guarantee of its the performance.

Our directions for future research are as follows.
In addition to active perception, active ``learning'' for multimodal categorization is also an important research topic. It takes a longer time for a robot to gather multimodal information to form  multimodal object categories from a massive number of daily objects than it does to recognize a new object. If a robot can notice that ``the object is obviously a sample of learned category,'' the robot need not obtain knowledge about object categories from such an object. In contrast, if a target object appears to be completely new to the robot, the robot should carefully interact with the object to obtain multimodal information from the object. Such a scenario will be achieved by developing an active ``learning'' method for multimodal categorization. It is likely that such a method  will be able to be obtained by extending our proposed active perception method. 

In addition, the MHDP model treated in this paper assumed that an action for perception is related to only one modality, e.g., grasping only corresponds to $m^h$. However, in reality, when we interact with an object with a specific action, e.g., grasping, shaking, or hitting, we obtain rich information related to various modalities. For example, when we shake a box to obtain auditory information, we also unwittingly obtain haptic information and information about its weight. The tight linkage between the modality information and an action is a type of approximation taken in this research. An extension of our model and the MHDP to a model that can treat actions that are related to various modalities is also a task for our future work.

\appendix
%\onecolumn
\section{Proof of the Optimality of the Proposed Active Perception Strategy}
In this appendix, we show that the proposed active perception strategy, which maximizes the expected KL divergence between the current state and the posterior distribution of $z_j$ after a selected set of actions, minimizes the expected KL divergence between the next and final states. 
\begin{eqnarray}
&&\argmin_{\mathbf{A} \in \mathbf{F}^{\mathbf{m_o}_j}_L}
	\mathbb{E}_{ X_j^{\mathbf{M}\setminus \mathbf{m_o}_j} | X^{\mathbf{m_o}_j}_j }
		\big[ \KL \big( P( z_j | X_j^{\mathbf{M}} ) , P( z_j | X_j^{\mathbf{A} \cup \mathbf{m_o}_j} ) \big)  \big] \nonumber\\
&=& \argmin_{\mathbf{A} \in \mathbf{F}^{\mathbf{m_o}_j}_L}
	\sum_{X_j^{\mathbf{M}\setminus \mathbf{m_o}_j}}\sum_{z_j}
		\big[P(X_j^{\mathbf{M}\setminus \mathbf{m_o}_j} | X^{\mathbf{m_o}_j}_j)
            P(z_j|X_j^\mathbf{M}) \log \frac{P(z_j|X_j^\mathbf{M})}{P(z_j |X^{\mathbf{m_o}_j}_j,X_j^\mathbf{A})} \big] \label{eq:ap1}
\end{eqnarray}
The numerator inside of the log function does not depend on $\mathbf{A} $. Therefore, the term related to the numerator can be deleted. In addition, by negating the remaining term, we obtain			
\begin{eqnarray}
(\ref{eq:ap1})&=& \argmax_{\mathbf{A} \in \mathbf{F}^{\mathbf{m_o}_j}_L}
	\sum_{X_j^{\mathbf{M}\setminus \mathbf{m_o}_j}}\sum_{z_j}
		[P(X_j^{\mathbf{M}\setminus \mathbf{m_o}_j} | X^{\mathbf{m_o}_j}_j)
        P(z_j|X_j^\mathbf{M}) \log  P(z_j |X^{\mathbf{m_o}_j}_j,X_j^\mathbf{A} )] \nonumber\\
&=& \argmax_{\mathbf{A} \in \mathbf{F}^{\mathbf{m_o}_j}_L}
	\sum_{X_j^{\mathbf{M}\setminus \mathbf{m_o}_j}}\sum_{z_j}
		[P(z_j, X_j^{\mathbf{M}\setminus \mathbf{m_o}_j} | X^{\mathbf{m_o}_j}_j)
        \log P(z_j |X^{\mathbf{m_o}_j}_j,X_j^\mathbf{A} )]. \label{eq:ap2}
\end{eqnarray}			
By marginalizing $X_j^{\mathbf{M}\setminus (\mathbf{m_o}_j \cup \mathbf{A})}$ from (\ref{eq:ap2}), we obtain
\begin{align}		
(\ref{eq:ap2}) =& \argmax_{\mathbf{A} \in \mathbf{F}^{\mathbf{m_o}_j}_L}
	\sum_{X_j^\mathbf{A}}\sum_{z_j}
		P(z_j, X_j^\mathbf{A} | X^{\mathbf{m_o}_j}_j)
			\log  P(z_j |X^{\mathbf{m_o}_j}_j,X_j^\mathbf{A} )\nonumber\\
= &\argmax_{\mathbf{A} \in \mathbf{F}^{\mathbf{m_o}_j}_L}
	[ \big( \sum_{X_j^\mathbf{A}}\sum_{z_j}
		P(z_j, X_j^\mathbf{A} | X^{\mathbf{m_o}_j}_j)
			\log P(z_j |X^\mathbf{m_o}_j,X_j^\mathbf{A} ) \big)
			\nonumber\\&~~\times \big( - \underbrace{\sum_{z_j} P(z_j|X_j^{\mathbf{m_o}_j}) \log P(z_j | X_j^{\mathbf{m_o}_j})}_{constant\ w.r.t.\ \mathbf{A}} \big) ]\nonumber\\
=& \argmax_{\mathbf{A} \in \mathbf{F}^{\mathbf{m_o}_j}_L}
	\big[\sum_{X_j^\mathbf{A}}\sum_{z_j}
		[P(X^\mathbf{A}_j | X_j^{\mathbf{m_o}_j}) P(z_j|X_j^{\mathbf{m_o}_j}, X^\mathbf{A}_j)
            \log  P(z_j |X_j^{\mathbf{m_o}_j},X_j^\mathbf{A} )]
		\nonumber\\&~~	- \sum_{X_j^\mathbf{A}}\sum_{z_j}
			[P(X^\mathbf{A}_j | X_j^{\mathbf{m_o}_j}) P(z_j|X_j^{\mathbf{m_o}_j}, X^\mathbf{A}_j)
			\log P(z_j | X_j^{\mathbf{m_o}_j})]\big] \nonumber\\
=& \argmax_{\mathbf{A} \in \mathbf{F}^{\mathbf{m_o}_j}_L}
	\sum_{X_j^\mathbf{A}}\sum_{z_j}
		[ P(X_j^\mathbf{A} | X^{\mathbf{m_o}_j}_j)
			 \KL \big( P(z_j |X^{\mathbf{m_o}_j}_j,X_j^\mathbf{A} ), P(z_j | X_j^{\mathbf{m_o}_j}) \big)] \nonumber\\
=& \argmax_{\mathbf{A} \in \mathbf{F}^{\mathbf{m_o}_j}_L}  \mathbb{E}_{ X_j^{\mathbf{A} } | X_j^{\mathbf{m_o}_j} } [ \KL \big( P( z_j | X_j^{ \mathbf{A} \cup \mathbf{m_o}_j})  , P( z_j | X_j^{\mathbf{m_o}_j}) \big) ].\nonumber
\end{align}

\section*{Acknowledgment}
The authors would like to thank undergraduate student Takuya Takeshita and graduate student Hajime Fukuda of Ritsumeikan University, who helped us develop the experimental instruments for obtaining our preliminary results.
This research was partially supported by a Grant-in-Aid for Young Scientists (B) 2012-2014 (24700233) funded by the Ministry of Education, Culture, Sports, Science, and Technology, Japan,  Tateishi Science and Technology Foundation, and JST, CREST.

\vskip 0.2in
\bibliography{refs}

\begin{thebibliography}{53}
\providecommand{\natexlab}[1]{#1}
\providecommand{\url}[1]{\texttt{#1}}
\expandafter\ifx\csname urlstyle\endcsname\relax
  \providecommand{\doi}[1]{doi: #1}\else
  \providecommand{\doi}{doi: \begingroup \urlstyle{rm}\Url}\fi

\bibitem[Ando et~al.(2013)Ando, Nakamura, Araki, and Nagai]{Ando}
Yoshiki Ando, Tomoaki Nakamura, Takaya Araki, and Takayuki Nagai.
\newblock Formation of hierarchical object concept using hierarchical latent
  dirichlet allocation.
\newblock In \emph{IEEE/RSJ International Conference on Intelligent Robots and
  Systems}, pages 2272--2279, 2013.

\bibitem[Araki et~al.(2012)Araki, Nakamura, Nagai, Nagasaka, Taniguchi, and
  Iwahashi]{Araki2012}
Takaya Araki, Tomoaki Nakamura, Takayuki Nagai, Shogo Nagasaka, Tadahiro
  Taniguchi, and Naoto Iwahashi.
\newblock {Online learning of concepts and words using multimodal LDA and
  hierarchical Pitman-Yor Language Model}.
\newblock In \emph{IEEE/RSJ International Conference on Intelligent Robots and
  Systems}, pages 1623--1630, 2012.

\bibitem[Asada et~al.(2009)Asada, Hosoda, Kuniyoshi, Ishiguro, Inui, Yoshikawa,
  Ogino, and Yoshida]{Asada2009a}
Minoru Asada, Koh Hosoda, Yasuo Kuniyoshi, Hiroshi Ishiguro, Toshio Inui,
  Yuichiro Yoshikawa, Masaki Ogino, and Chisato Yoshida.
\newblock {Cognitive Developmental Robotics: A Survey}.
\newblock \emph{IEEE Transactions on Autonomous Mental Development}, 1\penalty0
  (1):\penalty0 12--34, 2009.

\bibitem[Barsalou(1999)]{Barsalou1999}
Lawrence~W. Barsalou.
\newblock {Perceptual symbol systems}.
\newblock \emph{Behavioral and Brain Sciences}, 22\penalty0 (04):\penalty0
  1--16, 1999.
\newblock ISSN 0140-525X.
\newblock \doi{10.1017/S0140525X99002149}.

\bibitem[Blei et~al.(2003)Blei, Ng, and Jordan]{Blei2003}
David~M Blei, Andrew~Y Ng, and Michael~I Jordan.
\newblock Latent dirichlet allocation.
\newblock \emph{the Journal of machine Learning research}, 3:\penalty0
  993--1022, 2003.

\bibitem[Borotschnig et~al.(2000)Borotschnig, Paletta, Prantl, and
  Pinz]{Borotschnig2000}
H~Borotschnig, L~Paletta, M~Prantl, and A~Pinz.
\newblock {Appearance-based active object recognition}.
\newblock \emph{Image and Vision Computing}, 18:\penalty0 715--727, 2000.

\bibitem[Burgard et~al.(1997)Burgard, Fox, and Thrun]{Fox}
Wolfram Burgard, Dieter Fox, and Sebastian Thrun.
\newblock {Active Mobile Robot Localization}.
\newblock In \emph{Proceedings of the Fourteenth International Joint Conference
  on Artificial Intelligence (IJCAI)}, pages 1346--1352, 1997.

\bibitem[Cangelosi and Schlesinger(2015)]{Cangelosi2015}
Angelo Cangelosi and Matthew Schlesinger.
\newblock \emph{{Developmental Robotics}}.
\newblock The MIT press, 2015.

\bibitem[Celikkanat et~al.(2014)Celikkanat, Orhan, Pugeault, Guerin, Erol, and
  Kalkan]{Celikkanat}
Hande Celikkanat, Guner Orhan, Nicolas Pugeault, Frank Guerin, Sahin Erol, and
  Sinan Kalkan.
\newblock {Learning and Using Context on a Humanoid Robot Using Latent
  Dirichlet Allocation}.
\newblock In \emph{Joint IEEE International Conferences on Development and
  Learning and Epigenetic Robotics (ICDL-Epirob)}, pages 201--207, 2014.

\bibitem[Cohn et~al.(1996)Cohn, Ghahramani, and Jordan]{Cohn1996}
David~a. Cohn, Zoubin Ghahramani, and Michael~I. Jordan.
\newblock {Active learning with statistical models}.
\newblock \emph{Journal of Artificial Intelligence Research}, 4:\penalty0
  129--145, 1996.

\bibitem[Denzler and Brown(2002)]{Denzler2002}
Joachim Denzler and Christopher~M Brown.
\newblock {Information Theoretic Sensor Data Selection for Active Object
  Recognition and State Estimation}.
\newblock \emph{IEEE Transactions on pattern analysis and machine
  intelligence}, 24\penalty0 (2):\penalty0 1--13, 2002.

\bibitem[{Dutta Roy} et~al.(2004){Dutta Roy}, Chaudhury, and
  Banerjee]{DuttaRoy2004a}
Sumantra {Dutta Roy}, Santanu Chaudhury, and Subhashis Banerjee.
\newblock {Active recognition through next view planning: a survey}.
\newblock \emph{Pattern Recognition}, 37\penalty0 (3):\penalty0 429--446, 2004.

\bibitem[Eidenberger and Scharinger(2010)]{Eidenberger2010a}
R.~Eidenberger and J.~Scharinger.
\newblock {Active perception and scene modeling by planning with probabilistic
  6D object poses}.
\newblock In \emph{IEEE/RSJ International Conference on Intelligent Robots and
  Systems}, pages 1036--1043, 2010.

\bibitem[Fishel and Loeb(2012)]{Fishel2012}
Jeremy~A. Fishel and Gerald~E. Loeb.
\newblock {Bayesian exploration for intelligent identification of textures}.
\newblock \emph{Frontiers in Neurorobotics}, 6:\penalty0 1--20, 2012.
\newblock ISSN 16625218.

\bibitem[Gouko et~al.(2013)Gouko, Kobayashi, and Kim]{Gouko}
Manabu Gouko, Yuichi Kobayashi, and Chyon~Hae Kim.
\newblock {Online Exploratory Behavior Acquisition of Mobile Robot Based on
  Reinforcement Learning}.
\newblock In \emph{Recent Trends in Applied Artificial Intelligence}, pages
  272--281. 2013.

\bibitem[Griffith et~al.(2012)Griffith, Sinapov, Sukhoy, and
  Stoytchev]{Griffith2012}
Shane Griffith, Jivko Sinapov, Vladimir Sukhoy, and Alexander Stoytchev.
\newblock {A behavior-grounded approach to forming object categories:
  Separating containers from noncontainers}.
\newblock \emph{IEEE Transactions on Autonomous Mental Development}, 4\penalty0
  (1):\penalty0 54--69, 2012.

\bibitem[Harnad(1990)]{harnad1990symbol}
Stevan Harnad.
\newblock The symbol grounding problem.
\newblock \emph{Physica D: Nonlinear Phenomena}, 42\penalty0 (1):\penalty0
  335--346, 1990.

\bibitem[Hogman et~al.(2013)Hogman, Bjorkman, and Kragic]{Hogman2013}
Virgile Hogman, Mats Bjorkman, and Danica Kragic.
\newblock Interactive object classification using sensorimotor contingencies.
\newblock In \emph{IEEE/RSJ International Conference on Intelligent Robots and
  Systems (IROS)}, pages 2799--2805, 2013.

\bibitem[Ivaldi et~al.(2014)Ivaldi, Nguyen, Lyubova, Droniou, Padois, Filliat,
  Oudeyer, and Sigaud]{Ivaldi2014a}
Serena Ivaldi, Sao~Mai Nguyen, Natalia Lyubova, Alain Droniou, Vincent Padois,
  David Filliat, Pierre-Yves Oudeyer, and Olivier Sigaud.
\newblock {Object Learning Through Active Exploration}.
\newblock \emph{IEEE Transactions on Autonomous Mental Development}, 6\penalty0
  (1):\penalty0 56--72, 2014.

\bibitem[Iwahashi et~al.(2010)Iwahashi, Sugiura, Taguchi, Nagai, and
  Taniguchi]{Iwahashi2010}
Naoto Iwahashi, Komei Sugiura, Ryo Taguchi, Takayuki Nagai, and Tadahiro
  Taniguchi.
\newblock {Robots That Learn to Communicate: A Developmental Approach to
  Personally and Physically Situated Human-Robot Conversations}.
\newblock In \emph{Dialog with Robots Papers from the AAAI Fall Symposium},
  pages 38--43, 2010.

\bibitem[Ji and Carin(2006)]{Ji2006}
Shihao Ji and Lawrence Carin.
\newblock {Cost-Sensitive Feature Acquisition and Classification}.
\newblock \emph{Pattern Recognition}, 40\penalty0 (5):\penalty0 1474--1485,
  2006.

\bibitem[Krainin et~al.(2011)Krainin, Curless, and Fox]{Krainin2011a}
Michael Krainin, Brian Curless, and Dieter Fox.
\newblock {Autonomous generation of complete 3D object models using next best
  view manipulation planning}.
\newblock In \emph{IEEE International Conference on Robotics and Automation},
  pages 5031--5037, 2011.

\bibitem[Krause and Guestrin(2005)]{Krause05}
Andreas Krause and Carlos~E. Guestrin.
\newblock {Near-optimal Nonmyopic Value of Information in Graphical Models}.
\newblock In \emph{Proceedings of the Twenty-First Conference on Uncertainty in
  Artificial Intelligence}, 2005.

\bibitem[Lowe(2004)]{SIFT}
David~G Lowe.
\newblock Distinctive image features from scale-invariant keypoints.
\newblock \emph{International journal of computer vision}, 60\penalty0
  (2):\penalty0 91--110, 2004.

\bibitem[MacKay(2003)]{MacKay2003}
David J.~C. MacKay.
\newblock \emph{Information Theory, Inference and Learning Algorithms}.
\newblock Cambridge University Press, 2003.

\bibitem[Minoux(1978)]{Minoux78}
Michel Minoux.
\newblock Accelerated greedy algorithms for maximizing submodular set
  functions.
\newblock In \emph{Optimization Techniques}, pages 234--243. Springer, 1978.

\bibitem[Muslea et~al.(2006)Muslea, Minton, and Knoblock]{Muslea2006}
Ion Muslea, Steven Minton, and Craig~a. Knoblock.
\newblock {Active learning with multiple views}.
\newblock \emph{Journal of Artificial Intelligence Research}, 27\penalty0
  (1):\penalty0 203--233, 2006.

\bibitem[Nakamura et~al.(2007)Nakamura, Nagai, and Iwahashi]{Nakamura2007}
Tomoaki Nakamura, Takayuki Nagai, and Naoto Iwahashi.
\newblock {Multimodal object categorization by a robot}.
\newblock In \emph{IEEE/RSJ International Conference on Intelligent Robots and
  Systems}, pages 2415--2420, 2007.

\bibitem[Nakamura et~al.(2009)Nakamura, Nagai, and Iwahashi]{Nakamura2009}
Tomoaki Nakamura, Takayuki Nagai, and Naoto Iwahashi.
\newblock {Grounding of word meanings in multimodal concepts using LDA}.
\newblock In \emph{IEEE/RSJ International Conference on Intelligent Robots and
  Systems}, pages 3943--3948, 2009.

\bibitem[Nakamura et~al.(2011{\natexlab{a}})Nakamura, Nagai, and
  Iwahashi]{MHDP}
Tomoaki Nakamura, Takayuki Nagai, and Naoto Iwahashi.
\newblock Multimodal categorization by hierarchical dirichlet process.
\newblock In \emph{IEEE/RSJ International Conference on Intelligent Robots and
  Systems}, pages 1520--1525, 2011{\natexlab{a}}.

\bibitem[Nakamura et~al.(2011{\natexlab{b}})Nakamura, Nagai, and
  Iwahashi]{Nakamura2011}
Tomoaki Nakamura, Takayuki Nagai, and Naoto Iwahashi.
\newblock {Bag of multimodal LDA models for concept formation}.
\newblock \emph{IEEE International Conference on Robotics and Automation},
  pages 6233--6238, 2011{\natexlab{b}}.

\bibitem[Nakamura et~al.(2014)Nakamura, Nagai, Funakoshi, Nagasaka, Taniguchi,
  and Iwahashi]{Nakamura2014}
Tomoaki Nakamura, Takayuki Nagai, Kotaro Funakoshi, Shogo Nagasaka, Tadahiro
  Taniguchi, and Naoto Iwahashi.
\newblock {Mutual Learning of an Object Concept and Language Model Based on
  MLDA and NPYLM}.
\newblock In \emph{IEEE/RSJ International Conference on Intelligent Robots and
  Systems (IROS'14)}, pages 600 -- 607, 2014.

\bibitem[Natale et~al.(2004)Natale, Metta, and Sandini]{Natale}
Lorenzo Natale, Giorgio Metta, and Giulio Sandini.
\newblock {Learning haptic representation of objects}.
\newblock In \emph{{International Conference of Intelligent Manipulation and
  Grasping},}, 2004.

\bibitem[Nemhauser et~al.(1978)Nemhauser, Wolsey, and Fisher]{Nemhauser78}
George~L Nemhauser, Laurence~A Wolsey, and Marshall~L Fisher.
\newblock An analysis of approximations for maximizing submodular set
  functions-{I}.
\newblock \emph{Mathematical Programming}, 14\penalty0 (1):\penalty0 265--294,
  1978.

\bibitem[Pape et~al.(2012)Pape, Oddo, Controzzi, Cipriani, F\"{o}rster,
  Carrozza, and Schmidhuber]{Pape2012}
Leo Pape, Calogero~M. Oddo, Marco Controzzi, Christian Cipriani, Alexander
  F\"{o}rster, Maria~C. Carrozza, and J\"{u}rgen Schmidhuber.
\newblock {Learning tactile skills through curious exploration}.
\newblock \emph{Frontiers in Neurorobotics}, 6:\penalty0 1--16, 2012.

\bibitem[Pfeifer and Scheier(2001)]{Pfeifer2001}
Rolf Pfeifer and Christian Scheier.
\newblock \emph{Understanding Intelligence}.
\newblock A Bradford Book, 2001.
\newblock ISBN 9780262661256.

\bibitem[Rebguns et~al.(2011)Rebguns, Ford, and Fasel]{Rebguns2011}
Antons Rebguns, Daniel Ford, and Ian Fasel.
\newblock {InfoMax Control for Acoustic Exploration of Objects by a Mobile
  Robot}.
\newblock In \emph{AAAI11 Workshop on Lifelong Learning}, pages 22--28, 2011.
\newblock ISBN 9781577355311.

\bibitem[Roy and Pentland(2002)]{Roy2002a}
Deb~K. Roy and Alex~P. Pentland.
\newblock {Learning words from sights and sounds: a computational model}.
\newblock \emph{Cognitive Science}, 26\penalty0 (1):\penalty0 113--146, 2002.

\bibitem[Roy and Thrun(1999)]{Roy}
Nicholas Roy and Sebastian Thrun.
\newblock {Coastal Navigation with Mobile Robots}.
\newblock In \emph{Advances in Neural Processing Systems 12}, 1999.

\bibitem[Russo and Roy(2015)]{Russo2015}
Daniel Russo and Benjamin~Van Roy.
\newblock An information-theoretic analysis of thompson sampling, 2015.
\newblock arXiv:1403.5341v2.

\bibitem[Saegusa et~al.(2011)Saegusa, Natale, Metta, and Sandini]{Saegusa2011}
Ryo Saegusa, Lorenzo Natale, Giorgio Metta, and Giulio Sandini.
\newblock {Cognitive Robotics - Active Perception of the Self and Others -}.
\newblock In \emph{the 4th International Conference on Human System
  Interactions (HSI)}, pages 419--426, 2011.

\bibitem[Schneider et~al.(2009)Schneider, Sturm, Stachniss, Reisert, Burkhardt,
  and Burgard]{Schneider2009a}
Alexander Schneider, Jurgen Sturm, Cyrill Stachniss, Marco Reisert, Hans
  Burkhardt, and Wolfram Burgard.
\newblock {Object identification with tactile sensors using bag-of-features}.
\newblock In \emph{IEEE/RSJ International Conference on Intelligent Robots and
  Systems}, pages 243--248, 2009.

\bibitem[Settles(2012)]{Settles2012}
Burr Settles.
\newblock Active learning.
\newblock \emph{Synthesis Lectures on Artificial Intelligence and Machine
  Learning}, 6\penalty0 (1):\penalty0 1--114, 2012.

\bibitem[Sinapov and Stoytchev(2011)]{Sinapov}
Jivko Sinapov and Alexander Stoytchev.
\newblock {Object Category Recognition by a Humanoid Robot Using
  Behavior-Grounded Relational Learning}.
\newblock In \emph{IEEE International Conference on Robotics and Automation
  (ICRA)}, pages 184 -- 190, 2011.

\bibitem[Sinapov et~al.(2014)Sinapov, Schenck, Staley, Sukhoy, and
  Stoytchev]{Sinapov2014}
Jivko Sinapov, Connor Schenck, Kerrick Staley, Vladimir Sukhoy, and Alexander
  Stoytchev.
\newblock {Grounding semantic categories in behavioral interactions:
  Experiments with 100 objects}.
\newblock \emph{Robotics and Autonomous Systems}, 62\penalty0 (5):\penalty0
  632--645, 2014.

\bibitem[Stachniss et~al.(2005)Stachniss, Grisetti, and Burgard]{Stachniss}
C.~Stachniss, G.~Grisetti, and W.~Burgard.
\newblock {Information Gain-based Exploration Using Rao-Blackwellized Particle
  Filters}.
\newblock In \emph{Robotics Science and Systems (RSS)}, 2005.

\bibitem[Sudderth et~al.(2005)Sudderth, Torralba, Freeman, and
  Willsky]{Sudderth2006}
E.~B. Sudderth, A.~Torralba, W.~Freeman, and A.~S. Willsky.
\newblock {Describing Visual Scenes using Transformed Dirichlet Processes}.
\newblock In \emph{Advances in Neural Information Processing Systems},
  volume~18, pages 1297--1304, 2005.

\bibitem[Sushkov and Sammut(2012)]{Sushkov2012}
Oleg~O Sushkov and Claude Sammut.
\newblock Active robot learning of object properties.
\newblock In \emph{Intelligent Robots and Systems (IROS), 2012 IEEE/RSJ
  International Conference on}, pages 2621--2628. IEEE, 2012.

\bibitem[Taniguchi et~al.(2015)Taniguchi, Nagai, Nakamura, Iwahashi, Ogata, and
  Asoh]{Taniguchi2015SER}
Tadahiro Taniguchi, Takayuki Nagai, Tomoaki Nakamura, Naoto Iwahashi, Tetsuya
  Ogata, and Hideki Asoh.
\newblock Symbol emergence in robotics: A survey, 2015.
\newblock arXiv:1509.08973.

\bibitem[Teh et~al.(2006)Teh, Jordan, Beal, and Blei]{hdp}
Y.W. Teh, M.I. Jordan, M.J. Beal, and D.M. Blei.
\newblock {Hierarchical {Dirichlet} processes}.
\newblock \emph{Journal of the American Statistical Association}, 101\penalty0
  (476):\penalty0 1566--1581, 2006.

\bibitem[Tuci et~al.(2010)Tuci, Massera, and Nolfi]{Tuci2010}
Elio Tuci, Gianluca Massera, and Stefano Nolfi.
\newblock {Active Categorical Perception of Object Shapes in a Simulated
  Anthropomorphic Robotic Arm}.
\newblock \emph{IEEE Transactions on Evolutionary Computation}, 14\penalty0
  (6):\penalty0 885--899, 2010.

\bibitem[van Hoof et~al.(2012)van Hoof, Kroemer, {Ben Amor}, and
  Peters]{VanHoof2012a}
Herke van Hoof, Oliver Kroemer, Heni {Ben Amor}, and Jan Peters.
\newblock {Maximally informative interaction learning for scene exploration}.
\newblock In \emph{IEEE/RSJ International Conference on Intelligent Robots and
  Systems}, pages 5152--5158, 2012.

\bibitem[Velez et~al.(2012)Velez, Hemann, Huang, Posner, and Roy]{Velez2012}
Javier Velez, Garrett Hemann, Albert~S. Huang, Ingmar Posner, and Nicholas Roy.
\newblock {Modelling observation correlations for active exploration and robust
  object detection}.
\newblock \emph{Journal of Artificial Intelligence Research}, 44:\penalty0
  423--453, 2012.

\end{thebibliography}
\bibliographystyle{theapa}

\end{document}